\documentclass{article}

\PassOptionsToPackage{numbers,square,sort&compress}{natbib}
\usepackage[main, final]{neurips_2026}

\usepackage[utf8]{inputenc} 
\usepackage[T1]{fontenc}    
\usepackage{hyperref}       
\usepackage{url}            
\usepackage{booktabs}       
\usepackage{amsfonts}       
\usepackage{amsmath}        
\usepackage{nicefrac}       
\usepackage{microtype}      
\usepackage{xcolor}         
\usepackage{graphicx}       
\usepackage{multirow}       
\usepackage{array}          
\usepackage[ruled,vlined]{algorithm2e} 
\usepackage[most]{tcolorbox}   
\usepackage{changdefs}
\usepackage{wrapfig}
\usepackage{subcaption}

\newcommand{\old}{\text{old}}
\newcommand{\refm}{\text{ref}}
\newcommand{\base}{\text{base}}
\newcommand{\clip}{\text{clip}}

\title{Designing Reinforcement Learning for Diffusion Models: A Unified Path-Space View
}

\author{%
  \textbf{Yixian Xu}$^{1}$\thanks{Equal contribution.}  , \ \textbf{Yuanrui Zhang}$^{1*}$, \ \textbf{Shengjie Luo}$^{1\dagger}$, \ \textbf{Kaiyuan Gao}$^2$, \ \textbf{Liwei Wang}$^{1}$, \ \textbf{Di He}$^{1}$\thanks{Correspondence to: Shengjie Luo <\texttt{luoshengjie2408@gmail.com}>, Di He<\texttt{dihe@pku.edu.cn}>,.}\\
$^1$State Key Laboratory of General Artificial Intelligence,\\ 
School of Intelligence Science and Technology, Peking University \\
$^2$Huazhong University of Science and Technology \\
\texttt{\footnotesize \{xyx050,yuanruizhang25\}@stu.pku.edu.cn, luoshengjie2408@gmail.com}, \\
\texttt{\footnotesize  im\_kai@hust.edu.cn, \{wanglw,dihe\}@pku.edu.cn}
} 

\begin{document}

\maketitle

\vspace{-10pt}
\begin{abstract}
Reinforcement learning (RL) post-training provides a direct way to align diffusion models with human preferences and task-specific rewards. However, current RL algorithms for diffusion models remain fragmented: reverse-trajectory methods rely on discretized likelihood ratios, whereas forward-matching methods train on reward-labeled noising versions of the rollout samples. This paper shows that these seemingly different losses arise from a single path-space principle. Starting from the regularized diffusion-RL objective, we use importance sampling between sampling SDEs to obtain an explicit policy-gradient estimator on trajectory space. The estimator contains the stochastic It\^o integral underlying Flow-GRPO-type updates; we derive an equivalent variance-reduced value-gradient form that recovers the forward-matching structure of AWM and DiffusionNFT. This identifies the empirical gap between these method families as a variance-reduction effect rather than a difference in RL principle. The derivation yields a unified design space organized by value-gradient estimation, weight functions, and sampling choices. Within this space, we propose a multi-sample KDE value-gradient estimator that reuses rollout groups, together with scale-bounded weight families that retain stable existing recipes while excluding singular ones. Experiments on SD3.5-M and Qwen-Image models validate the variance-reduction explanation and show that the resulting recipe improves over prior diffusion-RL baselines. The code is available at \url{https://github.com/xyx050/DesigningDiffusionRL.git}.
\end{abstract}

\vspace{-6pt}
\section{Introduction}
\label{sec:intro}
\vspace{-2pt}

Diffusion and flow models have become the backbone of modern visual generation, powering high-quality image, video, and 3D synthesis~\citep{ho2020denoising,song2020score,lipman2022flow,liu2022flow,esser2024scaling}. The pretrained generator is optimized for distributional alignment via score matching, but not for the downstream preferences that users ultimately care about, including human aesthetic preference, compositional correctness, text rendering, or task-specific verifiability. Reinforcement learning (RL) post-training is therefore a natural next step that treats the generator as a policy and directly optimizes rewards evaluated on generated samples, including black-box or non-differentiable rewards that are hard to express as supervised losses.

This promise has led to a rapidly growing family of diffusion-RL algorithms. Starting from the raw objective of maximizing reward under the sampling policy with a reference regularizer, the central obstacle is that diffusion models do not provide the convenient policy likelihood ratios that PPO-style RL~\citep{schulman2017proximal} relies on. Different practical routes around this obstacle have been proposed: some discretize the reverse-sampling SDE so that each denoising step has a Gaussian transition ratio~\citep{black2023training,liu2025flow,wang2025grpo}; others avoid reverse-time likelihoods and instead build supervised flow-matching losses from reward-labeled clean samples and their forward-noised versions~\citep{xue2025advantage,zheng2025diffusionnft}. These routes produce loss functions that look structurally unrelated, and the community currently lacks a unified view for understanding existing recipes or designing principled diffusion-RL algorithms beyond combinations of heuristics.

Our first step is to expose the common continuous-time structure beneath these routes. Although a single diffusion path measure has no elementary likelihood, ratios between path measures generated by stochastic differential equations (SDEs) with the same diffusion coefficient have explicit form~\citep{anderson1982reverse,oksendal2013stochastic}. Applying this path-space importance sampling to the raw diffusion-RL objective yields a tractable estimator under a proposal path measure. Its policy gradient term contains a stochastic It\^o integral that corresponds to the noise exploration signal used by reverse-sampling-based GRPO methods~\citep{liu2025flow}. We prove that this term has an equivalent deterministic value-gradient representation.
This variance-reduced representation is the bridge from Flow-GRPO-style stochastic estimators to forward-matching methods: the large empirical gap between Flow-GRPO and DiffusionNFT/AWM is explained as a variance reduction effect rather than as a fundamentally different RL principle.

The variance-reduced form also reveals the design space. Its only unspecified component is how to estimate the value gradient $\nabla_{\bfxx_t}V_t$ on sampled trajectories, together with how to weight the resulting policy gradient and quadratic regularization terms across time. A Bayes rule expression for $\nabla_{\bfxx_t}V_t$ yields a deterministic one-sample estimator that recovers AWM and DiffusionNFT under their forward noising proposal, while the Brownian increment estimator recovers Flow-GRPO-type methods under reverse SDE rollouts. Lifting these choices to free parameters gives a unified template whose concrete entries specify the proposal dynamics, the value gradient estimator, the off-policy weight $\hat{w}_1$, the on-policy weight $\hat{w}_2$, and the sampler. These entries are organized into three axes: value gradient estimation, weight functions, and sampling. Under this template, previously separate algorithms become a specification of each entry in the same design space rather than separate design principles.

\looseness=-1Finally, the framework turns this taxonomy into concrete design guidance. Along the estimator axis, we observe that existing methods all use one-sample estimators of $\nabla_{\bfxx_t}V_t$; we introduce a multi-sample KDE estimator that reuses the rollout group already collected for GRPO-style advantages and reduces conditional variance without additional sampling. Along the weight axis, we identify a scale-bounded principle: the effective per-step update $\hat{w}_2(t)\widehat{\nabla V}_t$ and the quadratic regularizer $\hat{w}_1(t)\lrVert{\Delta\bfvv_\theta}^2$ should remain bounded across time and discretization scale. This principle recovers the stable weightings used by AWM, DiffusionNFT, and GRPO-Guard, rules out singular Flow-GRPO-style weights in off-policy regimes, and reduces algorithm design to a small family of interpretable weight shapes.

\paragraph{Contributions.} In summary, this paper contributes:
\begin{itemize}\setlength\itemsep{0.2em}
    \item \textbf{A continuous-time path-space framework (\secref{method-vr}, \secref{method-designspace}).} 
    Starting from the raw objective \eqnref{raw-obj}, we use path-space importance sampling and a variance-reduction theorem (\thmref{variance-reduction}) to derive an explicit template \eqnref{unified-loss} that unifies Flow-GRPO~\citep{liu2025flow}, Dance-GRPO~\citep{xue2025dancegrpo}, TempFlow-GRPO~\citep{he2025tempflow},GRPO-Guard~\citep{wang2025grpo},  AWM~\citep{xue2025advantage}, and DiffusionNFT~\citep{zheng2025diffusionnft} as special cases (\tabref{master}).
    \item \textbf{Principled design choices (\secref{method-principled}).} 
    We introduce a multi-sample KDE value-gradient estimator \eqnref{kde-estimator} and a scale-bounded weight principle that yields compact families for $\hat{w}_1$ and $\hat{w}_2$, recovering strong existing recipes while exposing the remaining tunable degrees of freedom.
    \item \textbf{Empirical validation (\secref{exp}).} On SD3.5-M~\citep{esser2024scaling} and Qwen-Image~\citep{wu2025qwen} models across PickScore~\citep{kirstain2023pick}, OCR~\citep{chen2023textdiffuser}, and GenEval~\citep{ghosh2023geneval} rewards, we verify that the variance-reduction trick closes the Flow-GRPO--DiffusionNFT efficiency gap, that the proposed weight families track empirical optima, and that our final recipe exceeds the baselines in \tabref{master}.
\end{itemize}

\section{Backgrounds}
\label{sec:background}
\subsection{Diffusion and Flow-Matching Generative Models}
\label{sec:background-fm}
Diffusion and Flow-Matching generative models perform step-by-step transformations between the data distribution and a tractable Gaussian prior. The forward process noises data to a Gaussian prior, and the reverse process generates data by denoising. Following the Flow-Matching formulation~\citep{lipman2022flow,liu2022flow,albergo2023stochastic}, the forward noising process is constructed as
\begin{align}\label{eqn:fm-forward-sde}
    \bfxx_t = (1-t)\bfxx_0+t\bfeps,\quad t\in[0,1], 
\end{align}
with $\bfeps\sim\clN(\bfzro,\bfI)$ and $\bfxx_0\sim p_{\text{data}}$.
The Flow-Matching formulation adopts the vector-field prediction loss to train the diffusion model:
\begin{align}
    \clL(\theta)=\bbE_{t\sim p(t), \bfxx_0\sim p_{\text{data}},\bfeps\sim\clN(\bfzro,\bfI)}\lrVert{\bfvv_\theta(\bfxx_t,t)-(\bfeps-\bfxx_0)}^2,
\end{align}
where $\bfvv_\theta(\bfxx_t,t)$ is parameterized by a neural network and $p(t)$ is the distribution of the time steps. Once the model is trained, the sampling process can be performed by simulating the following SDE:
\begin{align}\label{eqn:bg-samplingsde}
    \dd\bfxx_t = \left[\bfvv_\theta(\bfxx_t,t) + \frac{\eta_t}{1-t}\big(\bfxx_t+(1-t)\bfvv_\theta(\bfxx_t,t)\big)\right]\dd t + \sqrt{\frac{2t\eta_t}{1-t}}\,\dd\bfww_t,\quad\bfxx_1\sim\clN(\bfzro,\bfI),
\end{align}
where $\eta_t\geq0$ is a non-negative function and $\bfww_t$ is the standard Brownian motion on $\bbR^d$. When $\eta_t=1$, Eq.~\eqref{eqn:bg-samplingsde} is the time-reversal of Eq.~\eqref{eqn:fm-forward-sde}; when $\eta_t=0$, it collapses to the deterministic ODE $\frac{\dd \bfxx_t}{\dd t}=\bfvv_\theta(\bfxx_t,t)$, i.e.\ the FM analogue of the probability flow ODE~\citep{song2020score}.

\subsection{Reinforcement Learning for Diffusion Models}
\label{sec:background-rl}

Let $\clC$ be a prompt space and let $R:\bbR^d\times\clC\to[0,1]$ be a reward function evaluated on clean samples. We denote $\bfxx_{0:1}$ as the trajectory of the stochastic process $\{\bfxx_t\}_{t\in[0,1]}$ in \eqnref{bg-samplingsde} and $p(\bfxx_{0:1}, c)$ as its likelihood\footnote{This is a formal notation that treats the probability space defined on continuous paths as its discretization. Actually, this is known as the path measure corresponding to a stochastic differential equation (SDE). See \appxref{backgrounds} for background.}. RL post-training seeks to maximize the following canonical objective:

\begin{align}\label{eqn:raw-obj}
\clL(\theta)=\bbE_c\lrbrack{\bbE_{p_\theta(\bfxx_{[0:1]},c)}[R(\bfxx_0,c)]-\beta\KL(p_\theta\Vert p_\refm)},
\end{align}

where $p_\refm$ is a frozen reference policy.

Recently, several practical alternatives for \eqnref{raw-obj} have been proposed. Flow-GRPO~\citep{liu2025flow} applies an Euler--Maruyama discretization to the reverse-time sampling SDE, treating each denoising step as a tractable Gaussian transition and running clipped PPO over the discretized trajectory. Denote the current diffusion policy as $p_\theta$ and the old policy as $p_\old$. Given a group of old-policy trajectories $\{\{\bfxx_{t_j}^i\}_{j=1}^T\}_{i=1}^G$, Flow-GRPO optimizes:
\begin{align}\label{eqn:grpo-loss}
    \clL^{\mathrm{GRPO}}(\theta)=\bbE_{c, \{\bfxx_0^i\}_{i=1}^G\sim p_\old(\cdot|c)}\frac{1}{G}\sum_{i=1}^G\frac{1}{T}\sum_{j=1}^T\left(\min\left(r_j^i(\theta)A^i,\tilde{r}_j^i(\theta)A^i\right)-\beta\KL(p_\theta|p_\refm)\right).
\end{align}
Here $A^i=\frac{R(\bfxx_0^i,c)-\text{mean}(\{\bfxx_0^i\}_{i=1}^G)}{\text{std}(\{\bfxx_0^i\}_{i=1}^G)}$ is the group-normalized reward of the $i$-th sample in the group $\{\bfxx_0^i\}_{i=1}^G$, $\tilde{r}_j^i=\clip(r_j^i(\theta),1-\epsilon,1+\epsilon)$ is the clipped likelihood ratio of the $j$-th step for the $i$-th sample where $r_j^i=\frac{p_\theta(\bfxx_{t_{j-1}}^i|\bfxx_{t_{j}}^i,c)}{p_\old(\bfxx_{t_{j-1}}^i|\bfxx_{t_{j}}^i,c)}$ is Euler--Maruyama per-step likelihood ratio. 

A different line of methods forgoes reverse-trajectory likelihood ratios entirely and constructs training signals from reward-labeled clean samples via the forward noising path. AWM~\citep{xue2025advantage} samples $\bfxx_0$ from the old policy, evaluates its advantage $A(\bfxx_0,c)$, applies the forward noising path $\bfxx_t=(1-t)\bfxx_0+t\bfeps$, and reuses the flow-matching target with advantage weighting\footnote{Note that for AWM and DiffusionNFT, we also use $\clL$ to represent the maximum objective. }:
\begin{align}\label{eqn:bg-awm}
    - \clL^{\mathrm{AWM}}(\theta)
    =
    \bbE_{c,\bfxx_0,t,\bfeps}\!\left[
        A(\bfxx_0,c)\,
        \left\|
            \bfvv_\theta(\bfxx_t,t,c)-(\bfeps-\bfxx_0)
        \right\|^2
    \right].
\end{align}
DiffusionNFT~\citep{zheng2025diffusionnft} follows the same forward-noising
interface but uses the reward to define positive and negative velocity mixtures
around the old model, $\bfvv^{+}=(1-\beta)\bfvv_\old+\beta\bfvv_\theta$ and
$\bfvv^{-}=(1+\beta)\bfvv_\old-\beta\bfvv_\theta$:
\begin{equation}\label{eqn:loss-nft}
\resizebox{0.93\textwidth}{!}{$
\displaystyle
   - \clL^{\mathrm{NFT}}(\theta)
    =\bbE_{c,\bfxx_0,t,\bfeps}\left[
        R(\bfxx_0,c)\left\|\bfvv^{+}(\bfxx_t,t,c)-(\bfeps-\bfxx_0)\right\|^2
        +(1-R(\bfxx_0,c))
        \left\|\bfvv^{-}(\bfxx_t,t,c)-(\bfeps-\bfxx_0)\right\|^2
    \right].
$}
\end{equation}
At first glance, these methods appear to be driven by different principles and yield structurally unrelated loss functions. Without a unified view that connects them, the field still lacks a principled basis for understanding existing recipes or systematically designing new diffusion-RL algorithms.

\section{A Unified Path-Space View of Diffusion RL}
\label{sec:methods}
\subsection{A Path-Space View via Importance Sampling}
\label{sec:method-vr}

\paragraph{The structural obstacle of Diffusion RL.} Despite the variety of practical recipes, RL for diffusion models ultimately rests on the canonical objective \eqnref{raw-obj}. Its policy-gradient estimator $\nabla_\theta\bbE_{p_\theta}[R]$ formally requires the score $\nabla_\theta\log p_\theta(\bfxx_0\mid c)$, but $p_\theta(\bfxx_0\mid c)$, the marginal of the full generation trajectory, has no closed form. The trajectory likelihood $p_\theta(\bfxx_{[0:1]},c)$ fares no better: as the path measure of a continuous-time SDE, the path law is naturally defined as a measure on $C([0,1],\bbR^d)$, not as a tractable density with respect to a finite-dimensional reference measure~\citep{oksendal2013stochastic}. The log-derivative trick that underlies PPO/GRPO~\citep{schulman2017proximal,shao2024deepseekmath,zheng2025group} therefore does not apply directly. This is the structural obstacle that any RL algorithm on a diffusion generator must confront. Viewed through this perspective, the three methods of \secref{background-rl} are different workarounds to the same obstacle. Flow-GRPO discretises the reverse SDE into tractable Gaussian steps, while AWM and DiffusionNFT bypass reverse-time likelihoods entirely and match on the forward noising path.

\paragraph{Path-Space Importance Sampling.} 
We try to address the obstacle directly. The key observation is that although a single diffusion path measure has no density, \emph{pairs} of path measures generated by SDEs that share the same diffusion coefficient admit a closed-form importance sampling ratio~\citep{anderson1982reverse,oksendal2013stochastic}. Concretely, if we sample trajectories from a tractable proposal $q$ drawn from the same SDE family (\eqnref{bg-samplingsde}) with drift $\bfvv_\base$ in place of $\bfvv_\theta$ and the same noise schedule $\eta_t$, the ratio $p_\theta/q$ is explicit and differentiable in $\theta$ under certain conditions. 

This is exactly importance sampling in the \emph{path space}: we sample entire trajectories $\{\bfxx_t\}_{t\in[0,1]}$ rather than only the endpoint $\bfxx_0$. Applying this technique to \eqnref{raw-obj} and replacing the reward $R(\bfxx_0,c)$ with the group-relative advantage $A(\bfxx_0,c)$ (which leaves the policy gradient unchanged while reducing its variance), we obtain the unbiased estimator
\begin{align}\label{eqn:rl-is-loss}
    \clL(\theta,c) = \bbE_{q(\bfxx_{[0:1]},c)}\!\left[\frac{p_\theta(\bfxx_{[0:1]},c)}{q(\bfxx_{[0:1]},c)}A(\bfxx_0,c)\right]-\beta\bbE_{q(\bfxx_{[0:1]},c)}\!\left[\frac{p_\theta(\bfxx_{[0:1]},c)}{q(\bfxx_{[0:1]},c)}\log\frac{p_\theta(\bfxx_{[0:1]},c)}{p_\refm(\bfxx_{[0:1]},c)}\right],
\end{align}
where $p_\theta$, $p_\refm$, and $q$ correspond to the sampling SDE (\eqnref{bg-samplingsde}) with velocity fields $\bfvv_\theta$, $\bfvv_\refm$, and $\bfvv_\base$ respectively, and the same noise schedule $\eta_t$. In this way,
the importance sampling ratio $ M_1(\theta)\;:=\;\log\frac{p_\theta(\bfxx_{[0:1]},c)}{q(\bfxx_{[0:1]},c)}\;$ is given by (see \appxref{backgrounds}):
\begin{align}\label{eqn:girsanov-fm}
    M_1(\theta)=-\!\int_0^{1}\!\sqrt{2\wwt_{1-r}}\,\Delta\bfvv_\theta(\bfxx_{1-r},1-r)\cdot\dd\bfww_r\,-\,\int_0^{1}\!\wwt_{1-r}\,\lrVert{\Delta\bfvv_\theta(\bfxx_{1-r},1-r)}^2\dd r,
\end{align}
where $r=1-t$, $\Delta\bfvv_\theta:=\bfvv_\theta-\bfvv_\base$ and the weight $\wwt_t\;=\;\frac{(1+\eta_t)^2}{4\eta_t}\cdot\frac{1-t}{t}$
couples the training weight to the sampler noise schedule. Substituting $M_1(\theta)$ into \eqnref{rl-is-loss}, the objective becomes a single expectation under $q$:
\begin{align}\label{eqn:is-policy-loss}
    \clL(\theta,c) \;=\; \bbE_{q}\!\left[\exp(M_1(\theta))\,A(\bfxx_0,c)\right]\;-\;\beta\,\bbE_{q}\!\left[\exp(M_1(\theta))\,\big(M_1(\theta)-\log\tfrac{p_\refm(\bfxx_{[0:1]},c)}{q(\bfxx_{[0:1]},c)}\big)\right].
\end{align}

\paragraph{An explicit variance-reduced template.}
We focus here on the policy-gradient term of \eqnref{is-policy-loss}, $\clL^{\text{policy}}(\theta,c):=\bbE_q[\exp(M_1(\theta))A(\bfxx_0,c)]$; the KL term is handled identically. Under mild assumptions that follow standard PPO practice~\citep{schulman2017proximal,shao2024deepseekmath}, the explicit form of $\clL^{\text{policy}}(\theta,c)$ can be exactly obtained (See \appxref{proof-policy-grad} for proof):

\begin{proposition}[Stochastic-integral path-space estimator]\label{prop:policy-grad}
    Assume clipping or a small step size keeps $\exp(M_1(\theta^-)) \approx 1$\footnote{Although the IS ratio $\exp(M_1(\theta^-))$ has an expectation of $1$, it remains strictly below $1$ with high probability in high-dimensional latent spaces. See \appxref{is-ratio} for details.}, where $\theta^-$ denotes the stop-gradient version of the current parameters. Then
    \begin{align}\label{eqn:sde-estimator}
    -\clL^{\text{policy}}(\theta,c)\,=_\theta\,\bbE_q\!\left[\,\underbrace{A(\bfxx_0,c)\!\int_0^1\!\sqrt{2\wwt_{1-r}}\,\Delta\bfvv_\theta\cdot\dd\bfww_r}_{\text{Term (A): noise exploration}}+\underbrace{A(\bfxx_0,c)\!\int_0^1\!\wwt_{1-r}\,\lrVert{\Delta\bfvv_\theta}^2\dd r}_{\text{Term (B): regularisation}}\right],
\end{align}
where $=_\theta$ denotes equality up to a $\theta$-independent constant.
\end{proposition}
Intuitively, Term~(B) is an advantage-weighted quadratic that anchors $\bfvv_\theta$ to $\bfvv_\base$, while Term~(A) follows the Brownian increments weighted by advantage --- pushing $\bfvv_\theta$ toward the noise realisations that yielded higher rewards. The latter is an It\^o integral, however, whose variance dominates the overall variance of \eqnref{sde-estimator}. A natural question arises: can Term~(A) be rewritten as a deterministic integral with the same expectation and with smaller variance? 
The next theorem provides an explicit variance-reduced template (See \appxref{proof-vr} for the proof):

\begin{theorem}[Variance-reduced path-space estimator]\label{thm:variance-reduction}
Let $q$ be the path measure of the sampling SDE \eqnref{bg-samplingsde} with base drift $\bfvv_\base$ and noise schedule $\eta_t>0$, and let $V_t(\bfxx,c):=\bbE_q[A(\bfxx_0,c)\mid\bfxx_t=\bfxx]$ be the associated value function.
Under the regularity the trust-region condition of \propref{policy-grad}, the policy gradient of \eqnref{is-policy-loss} admits the (locally) variance-reduced representation
\begin{align}\label{eqn:vr-form}
\resizebox{0.94\textwidth}{!}{$
\displaystyle
-\clL^{\text{policy}}(\theta,c)= {}_\theta\bbE_q\left[\int_0^1\left((1+\eta_t)\Delta\bfvv_\theta(\bfxx_t,t)\cdot\nabla_{\bfxx_t}V_t(\bfxx_t,c)+A(\bfxx_0,c)\wwt_t\lrVert{\Delta\bfvv_\theta(\bfxx_t,t)}^2\right)\dd t\right].
$}
\end{align}
\end{theorem}

\looseness=-1From the intractable canonical objective \eqnref{raw-obj}, we arrive at the explicit template \eqnref{vr-form}, which is both \emph{(i) tractable} --- every term is a deterministic integral under a proposal $q$ in the same SDE family as $p_\theta$, with no reliance on path-likelihood densities of $p_\theta$ that do not exist --- and \emph{(ii) variance-reduced} --- the stochastic integral in \eqnref{sde-estimator} is replaced by a deterministic integral depending on a $\nabla_{\bfxx_t}V_t$ estimator. These two properties enable \eqnref{vr-form} to serve as a unified tool to understand different RL methods for the diffusion model.

\vspace{-2pt}
\subsection{A Unified Framework and Its Design Space}\label{sec:method-designspace}
\vspace{-2pt}

\paragraph{The missing ingredient.}
\eqnref{vr-form} leaves one quantity unspecified: the value gradient $\nabla_{\bfxx_t}V_t(\bfxx_t,c)$. How this quantity is evaluated on each rollout together with the weights on the two terms of \eqnref{vr-form} turns out to distinguish each diffusion RL method of \secref{background-rl} from others. We begin by deriving an analytical Bayes-rule expression for $\nabla_{\bfxx_t}V_t$, from which a one-sample estimator follows.

\begin{proposition}\label{prop:single-point}
The gradient of the value function admits the closed-form expression
\begin{align}
    \nabla_{\bfxx_t} V_t(\bfxx_t, c)=\bbE_{q(\bfxx_0\mid\bfxx_t)}\!\left[\nabla_{\bfxx_t} \log q(\bfxx_t\mid\bfxx_0)\,A(\bfxx_0,c)\right] - V_t(\bfxx_t,c)\,\bbE_{q(\bfxx_0\mid\bfxx_t)}\!\left[\nabla_{\bfxx_t} \log q(\bfxx_t\mid\bfxx_0)\right].
\end{align}
Let $\bfvv(\bfxx_t,\bfxx_0):=(\bfxx_t-\bfxx_0)/t$ denote the FM conditional velocity. Under $q$ with $\bfvv_\base=\bbE[\bfvv\mid\bfxx_t]\approx\bfvv_\old$ and $\eta_t\equiv 1$ (i.e.,~$q$ is the reverse of the forward noising process), this specializes in
\begin{align}\label{eqn:gradV-centered}
    \nabla_{\bfxx_t} V_t(\bfxx_t,c)=-\tfrac{1-t}{t}\,\bbE_{q(\bfxx_0\mid\bfxx_t)}\!\left[A(\bfxx_0,c)\,\big(\bfvv(\bfxx_t,\bfxx_0)-\bfvv_\base(\bfxx_t,t)\big)\right].
\end{align}
The corresponding one-sample estimator on a rollout trajectory $\bfxx_{0:1}$ is 
\begin{align}
\widehat{\nabla V}_t^{\,\mathrm{det}}(\bfxx_t,c):=-\tfrac{1-t}{t}A(\bfxx_0,c)\big(\bfvv(\bfxx_t,\bfxx_0)-\bfvv_\old(\bfxx_t,t)\big).
\end{align}
\end{proposition}

See \appxref{proof-single-point} for the proof. Substituting $\widehat{\nabla V}_t^{\,\mathrm{det}}$ into \thmref{variance-reduction} yields the concrete training objective
\begin{align}\label{eqn:temp-loss}
\resizebox{0.94\textwidth}{!}{$
\displaystyle
\bbE_q\!\left[\int_0^1\!\left(\underbrace{-\tfrac{(1-t)(1+\eta_t)}{t}}_{\text{weight term 1}}\Delta\bfvv_\theta(\bfxx_t,t)\cdot A(\bfxx_0,c)\big(\bfvv-\bfvv_\old(\bfxx_t,t)\big)+\underbrace{A(\bfxx_0,c)\,\wwt_t}_{\text{weight term 2}}\lrVert{\Delta\bfvv_\theta(\bfxx_t,t)}^2\right)\dd t\right],
$}
\end{align}
where $\bfvv=(\bfxx_t-\bfxx_0)/t$ is the single-sample conditional velocity. \eqnref{temp-loss} is a fully explicit, variance-controlled diffusion-RL loss obtained by instantiating the variance-reduced template \eqnref{vr-form} with the deterministic one-sample estimator of \propref{single-point} under the noise schedule $\eta_t\!=\!1$.

\paragraph{A unified framework.}
Moreover, the stochastic integrand of Term~(A) in \eqnref{sde-estimator} is itself an estimator of $\nabla_{\bfxx_t}V_t$ --- a \emph{stochastic} one-sample estimator using the Brownian increment of the rollout SDE, 
$\widehat{\nabla V}_t^{\,\mathrm{sto}}(\bfxx_t,c):=A(\bfxx_0,c)\sqrt{(1-t)/(2\eta_t t\Delta t)}\,\bfeps_t$. 
Substituting it through the structural identity of \thmref{variance-reduction} gives a second, equally valid instantiation. The two instantiations share a common two-term form: an off-policy quadratic in $\Delta\bfvv_\theta$ weighted by a function of the advantage, plus an on-policy inner product between $\Delta\bfvv_\theta$ and the chosen value-gradient estimator. 
Viewing the estimator and the two weights 
as free parameters yields the following unified framework:
\begin{tcolorbox}[colback=blue!5!white,colframe=black!40!white,boxrule=0.4pt,arc=1pt,left=5pt,right=5pt,top=4pt,bottom=4pt]
\begin{align}\label{eqn:unified-loss}
\resizebox{0.93\textwidth}{!}{$
\displaystyle
    -\clL^{\text{policy}}(\theta,c)=\int_0^1\bbE_q\!\Big[\underbrace{\hat{w}_1\!\big(A(\bfxx_0,c),t\big)\lrVert{\Delta\bfvv_\theta(\bfxx_t,t)}^2}_{\text{off-policy regulariser}}+\underbrace{\hat{w}_2(t)\big\langle\Delta\bfvv_\theta(\bfxx_t,t),\widehat{\nabla V}_t(\bfxx_t,c)\big\rangle}_{\text{on-policy policy-gradient}}\Big]\dd t,
$}
\end{align}
\end{tcolorbox}
\noindent where $\widehat{\nabla V}_t$ is any estimator of $\nabla_{\bfxx_t}V_t(\bfxx_t,c)$ and $\hat{w}_1,\hat{w}_2$ are arbitrary weight functions. In principle, different diffusion RL methods of \secref{background-rl} turn out to be instances of \eqnref{unified-loss} under specific choices of these parameters. \tabref{master} enumerates these choices for various methods.

\begin{tcolorbox}[colback=blue!5!white,colframe=black!40!white,boxrule=0.4pt,arc=1pt,left=5pt,right=5pt,top=4pt,bottom=4pt]
\paragraph{The design space.} Reading \tabref{master} by columns, the concrete choices in \eqnref{unified-loss} group naturally into three independent axes:
\begin{enumerate}
    \item[\textbf{(i)}] \textbf{Value-gradient estimation}: the proposal $q$ (determined by the choice of $\bfvv_\base$ and $\eta_t$) together with the estimator $\widehat{\nabla V}_t$ evaluated under $q$. A typical choice takes $q$ induced by a lagged velocity model $\bfvv_\old$ and a specific noise schedule $\eta_t$.
    \item[\textbf{(ii)}] \textbf{Weight functions}: the off-policy weight $\hat{w}_1(A,t)$ and the on-policy weight $\hat{w}_2(t)$. When $\bfvv_\old=\bfvv_\theta$ (on-policy), the off-policy regulariser of \eqnref{unified-loss} vanishes and only $\hat{w}_2$ matters; when $\bfvv_\old\neq\bfvv_\theta$ (off-policy), $\hat{w}_1$ and $\hat{w}_2$ jointly determine the training signal.
    \item[\textbf{(iii)}] \textbf{Sampler}: the numerical scheme used to simulate $q$ and collect rollout samples.
\end{enumerate}
\end{tcolorbox}

\begin{table}[t]
\centering\small
\caption{Diffusion RL methods as rows in the five-knob design space of \eqnref{unified-loss}. We use the FM parameterisation throughout; $\bfvv$ and $\bfeps$ denote the FM conditional velocity and the per-step Brownian increment. See \appxref{table-proof} for detailed derivations of each row.}
\label{tab:master}
\renewcommand{\arraystretch}{1.35}
\setlength{\tabcolsep}{4pt}
\resizebox{1.00\textwidth}{!}{\begin{tabular}{@{}lcccccc@{}}
\toprule
Method & $\bfvv_\base$ ($\eta_t$) & $\widehat{\nabla V}_t$ & $s(t)$ & $\hat{w}_1(A,t)$ & $\hat{w}_2(t)$ & Sampler \\
\midrule
Flow-GRPO~\citep{liu2025flow,xue2025dancegrpo}
  & $\bfvv_\old$ ($\eta_t$)
  & $s(t)\cdot A\bfeps$
  & $\sqrt{\!\tfrac{1-t}{2\eta_t t \Delta t}}$
  & $A\cdot\frac{(1+\eta_t)^2}{4\eta_t}\frac{1-t}{t}$
  & $1+\eta_t$
  & Flow-SDE\\
GRPO-Guard~\citep{wang2025grpo}
  & $\bfvv_\old$ ($\eta_t$)
  & $s(t)\cdot A\bfeps$
  & $\sqrt{\!\tfrac{1-t}{2\eta_t t \Delta t}}$
  & $0$
  & $\frac{(1+\eta_t)}{s(t)\Delta t}$
  & Flow-SDE\\
TempFlow-GRPO~\citep{he2025tempflow} 
  & $\bfvv_\old$ ($\eta_t$)
  & $s(t)\cdot A\bfeps$
  & $\sqrt{\!\tfrac{1-t}{2\eta_t t \Delta t}}$
  & $A\cdot\frac{(1+\eta_t)^2}{4}\sqrt{\frac{2(1-t)\Delta t}{\eta_tt}}$
  & $(1+\eta_t)/s(t)$
  & Branched-SDE\\
AWM~\citep{xue2025advantage}
  & $\bbE[\bfvv\mid\bfxx_t]$ ($1$)
  & $-s(t)\cdot A(\bfvv\!-\!\bfvv_\old)$
  & $\tfrac{1-t}{t}$
  & $A/\Delta t$
  & $\tfrac{2}{s(t)\Delta t}$
  & Flow-SDE / ODE\\
DiffusionNFT~\citep{zheng2025diffusionnft}
  & $\bbE[\bfvv\mid\bfxx_t]$ ($1$)
  & $-s(t)\cdot A(\bfvv\!-\!\bfvv_\old)$
  & $\tfrac{1-t}{t}$
  & $1/\Delta t$
  & $\tfrac{2}{\beta s(t)\Delta t}$
  & DPM-ODE~\citep{lu2025dpm}\\
\midrule
\textbf{Ours} (\secref{method-principled})
  & $\bfvv_\old$ ($\eta_t$)
  & $\widehat{\nabla V}_t^{\,\mathrm{KDE}}$(\eqnref{kde-estimator})
  & $\frac{1-t}{t}$
  & $(1-t)^{\alpha_1}/\Delta t$
  & $\frac{t^{\alpha_2}}{s(t)\Delta t}$
  & Flow-SDE\\
\bottomrule
\end{tabular}}
\vspace{-10pt}
\end{table}

\vspace{-2pt}
\subsection{Elucidating Design Choices}\label{sec:method-principled}

\looseness=-1The design space of \secref{method-designspace} has two primary axes: the value-gradient estimator $\widehat{\nabla V}_t$ and the weight pair $(\hat{w}_1,\hat{w}_2)$. We elucidate each axis in turn\footnote{We treat sampler choices as specific designs within value-gradient estimation (e.g., trajectory branching in MixGRPO and TempFlow-GRPO).}. For the estimator, we observe that methods in \tabref{master} rely on a one-sample Monte Carlo estimator of $\nabla_{\bfxx_t}V_t$ and we introduce a multi-sample refinement based on kernel density estimation. For the weights, we extract a scale-bounded structural principle from recent advanced methods and show that it implies the $t$-dependent shape of both $\hat{w}_1$ and $\hat{w}_2$ given the choice of value-gradient estimator.

\paragraph{Estimator axis: from one-sample to multi-sample.}
Every row of \tabref{master} uses a one-sample value gradient estimator:
GRPO-type methods use the original stochastic one-sample estimator $\widehat{\nabla V}_t^{\,\mathrm{sto}}$ built from a single Brownian increment, while AWM and DiffusionNFT use the deterministic one-sample estimator $\widehat{\nabla V}_t^{\,\mathrm{det}}$ (defined in \propref{single-point}) from a single $(\bfxx_0,\bfxx_t)$ pair. Both carry the full conditional variance of the integrand per rollout 
which remains unchanged by simply enlarging the group.

The natural variance reduction is to average over $K$ samples from $q(\bfxx_0|\bfxx_t)$, but drawing $K$ fresh samples at each query $\bfxx_t$ would cost $K$ extra rollouts. Instead, we exploit the collected rollout group $\{(\bfxx_0^i,\bfxx_t^i)\}_{i=1}^G$. 
For a noisy state $\bfxx_t$, we take a kernel-weighted average over the $G$ group samples with a kernel matched to the FM forward likelihood $q(\bfxx_t|\bfxx_0)=\clN(\bfxx_t;(1-t)\bfxx_0,t^2\bfI)$. This gives the Nadaraya--Watson KDE estimator
\begin{align}\label{eqn:kde-estimator}
    \widehat{\nabla V}_t^{\,\mathrm{KDE}}(\bfxx_t,c)\;:=\;-\tfrac{1-t}{t}\cdot\frac{\sum_{i=1}^G K(\bfxx_t,\bfxx_0^i)\,A(\bfxx_0^i,c)\,\big(\bfvv(\bfxx_t,\bfxx_0^i)-\bfvv_\old(\bfxx_t,t)\big)}{\sum_{i=1}^G K(\bfxx_t,\bfxx_0^i)},
\end{align}
where $K(\bfxx_t,\bfxx_0):=\exp(-\frac{\lrVert{\bfxx_t-(1-t)\bfxx_0}^2}{2h\,t^2})$ is the FM-matched Gaussian kernel and $h$ is a hyperparameter to balance the relative weight for the samples in the group, typically fixed as a constant depends on the dimension. The KDE estimator is also inspired by \citet{bertrand2026closed}, which shows that including one anchor posterior sample can make the self-normalized estimator unbiased. We summarize the properties of different value-gradient estimators as follows.

\begin{proposition}[Properties of the value-gradient estimators]\label{prop:estimator-unbiased}
Let $A(\bfxx_0,c)$ be square-integrable. Assume $\bfvv_\base=\bfvv_\old$, then the following arguments hold:
\begin{enumerate}
    \item[(i).] The discretized stochastic one-sample estimator is asymptotically unbiased in the continuous-time It\^o limit of \thmref{variance-reduction}: $\lim_{\Delta t\to0}\bbE\!\left[\widehat{\nabla V}_{t,\Delta t}^{\,\mathrm{sto}}\mid\bfxx_t\right]=\nabla_{\bfxx_t}V_t(\bfxx_t,c)$;
    \item[(ii).] Assume $\eta_t\equiv1$, $\bfvv_\old(\bfxx_t,t)=\bbE[\bfvv(\bfxx_t,\bfxx_0)\mid\bfxx_t]$, then the deterministic one-sample estimator is unbiased: $\bbE\!\left[\widehat{\nabla V}_t^{\,\mathrm{det}}\mid\bfxx_t\right]=\nabla_{\bfxx_t}V_t(\bfxx_t,c)$;
    \item[(iii).] Under the assumption of (ii) and $h=1$, the KDE estimator is unbiased and has smaller variance than the deterministic one-sample estimator:
    \begin{align}
        \bbE\!\left[\widehat{\nabla V}_t^{\,\mathrm{KDE}}(\bfxx_t^j,c)\mid\bfxx_t^j\right]
        =\nabla_{\bfxx_t}V_t(\bfxx_t^j,c),\quad
        \mathrm{Var}\!\left(\widehat{\nabla V}_t^{\,\mathrm{KDE}}\mid\bfxx_t^j\right)
        \leq \mathrm{Var}\!\left(\widehat{\nabla V}_t^{\,\mathrm{det}}\mid\bfxx_t^j\right).
    \end{align}
\end{enumerate}
\end{proposition}
See \appxref{proof-estimator-unbiased} for the proof. \propref{estimator-unbiased} shows that the KDE estimator is unbiased and performs variance reduction for the deterministic one-sample estimator when $h=1$. With the same rollout construction but $h\neq1$, the anchor is drawn from the true posterior while the normalized weights correspond to a different kernel, so the estimator is generally biased. Thus, $h$ explicitly trades exact posterior matching for smoother finite-sample estimates; we quantify this trade-off in \secref{var-reduction-exp}.

\paragraph{Weight axis: a scale-bounded principle.}
From \eqnref{unified-loss}, the on-policy gradient contribution at time $t$ has magnitude proportional to $\hat{w}_2(t)\,\widehat{\nabla V}_t(\bfxx_t,c)$: this product is what scales the per-step update before the network-gradient factor $\partial_\theta\Delta\bfvv_\theta$ enters. For stable training, its magnitude should be controlled across $t$: neither blowing up at particular timesteps (which would amplify gradient variance and skew the update) nor vanishing where gradient signal is needed. We therefore inspect $\hat{w}_2(t)\,\widehat{\nabla V}_t$ across the rows of \tabref{master}:
\begin{center}\small
\begin{tabular}{@{}l@{\hspace{1.5em}}l@{\hspace{1em}}l@{}}
Flow-GRPO: & $\hat{w}_2(t)\,\widehat{\nabla V}_t \Delta t \;=\; (1+\eta_t)\,A\,\sqrt{(1-t)\Delta t/(2\eta_t\,t)}\,\bfeps_t$ & \emph{(blows up as $,t\!\to\!0$)} \\
GRPO-Guard: & $\hat{w}_2(t)\,\widehat{\nabla V}_t\Delta t \;=\; (1+\eta_t)\,A\,\bfeps_t$ & \emph{(bounded)} \\
AWM: & $\hat{w}_2(t)\,\widehat{\nabla V}_t\Delta t \;=\; -2\,A\,(\bfvv-\bfvv_\old)$ & \emph{(bounded)} \\
DiffusionNFT: & $\hat{w}_2(t)\,\widehat{\nabla V}_t\Delta t \;=\; -\tfrac{2}{\beta}\,A\,(\bfvv-\bfvv_\old)$ & \emph{(bounded)}.
\end{tabular}
\end{center}
\noindent The three methods that empirically outperform Flow-GRPO each pair their estimator with a weight whose product $\hat{w}_2\,\widehat{\nabla V}_t\Delta t$ stays bounded uniformly in $t$; Flow-GRPO alone violates this property. We 
conclude the pattern to an explicit design principle:

\begin{tcolorbox}[colback=blue!5!white,colframe=black!40!white,boxrule=0.4pt,arc=1pt,left=5pt,right=5pt,top=4pt,bottom=4pt]
\textbf{Scale-bounded principle.} Choose $\hat{w}_2(t)$ so that $\hat{w}_2(t)\,\widehat{\nabla V}_t\Delta t$ stays bounded uniformly in $t$ ; analogously, choose $\hat{w}_1(A,t)$ so that $\hat{w}_1(A,t)\lrVert{\Delta\bfvv_\theta}^2\Delta t$ stays bounded uniformly in $t$.
\end{tcolorbox}

The principle applies asymmetrically to the two weights, reflecting the asymmetric $t$-dependence of the two factors it multiplies. For $\hat{w}_2(t)$, the deterministic estimator $\widehat{\nabla V}_t^{\,\mathrm{det}}$ and our KDE estimator $\widehat{\nabla V}_t^{\,\mathrm{KDE}}$ both inherit the $(1-t)/t$ prefactor of \propref{single-point}; the principle then \emph{pins down} the shape of $\hat{w}_2(t)$ up to a single free exponent. Let $s(t)$ be the scale factor of the estimator $\widehat{\nabla V}_t$ (summarized in \tabref{master}), $w_2(t):=\hat{w}_2(t)s(t)\Delta t$. 
The principle requires $w_2(t)$ to stay bounded uniformly in $t$. We adopt the minimal single-exponent family
    $w_2(t)\;=\;t^{\alpha}, \alpha\geq 0$
which saturates the principle 
and concentrates mass more sharply near $t=1$ for larger $\alpha$ (any multiplicative constant is omitted).

For $\hat{w}_1$, the situation is different: because $\lrVert{\Delta\bfvv_\theta}^2$ carries no $t$-singularity, the principle reduces to ``$\hat{w}_1\Delta t$ is bounded in $t$''. Flow-GRPO's $\hat{w}_1\Delta t=A\,\wwt_t\Delta t\!\propto\!A(1-t)\Delta t/t$ diverges as $t\!\to\!0$ and is ruled out; AWM's $\hat{w}_1\Delta t=A$ and DiffusionNFT's $\hat{w}_1\Delta t=1$ are admissible, but the principle alone does not pick the monotonicity direction in $t$ or the presence of advantage coupling. Let $w_1(A,t)=\hat{w}_1(A,t)\Delta t$. Four minimal single-exponent candidates remain:
\begin{align}\label{eqn:w1-family}
    w_1(A,t)\;\in\;\big\{\,t^{\alpha},\ (1-t)^{\alpha},\ A(\bfxx_0,c)\,t^{\alpha},\ A(\bfxx_0,c)\,(1-t)^{\alpha}\,\big\},\quad \alpha\geq 0,
\end{align}
distinguishing \textbf{(a)} whether the trust-region tightens toward the clean-data end ($(1-t)^\alpha$) or the noise end ($t^\alpha$), and \textbf{(b)} whether its strength couples to the advantage. Which of the four shapes is optimal is an empirical question the principle cannot resolve; we ablate across the family in \secref{mapping}.

\vspace{-2pt}
\section{Experiments}
\vspace{-2pt}
\label{sec:exp}
In this section, we show the experimental results on the design of each component in our framework. In \secref{var-reduction-exp}, we introduce the settings of our experiments and give an example to show how the variance reduction trick in \secref{method-vr} works. In \secref{mapping}, we show the principles of weight design for different estimators, which gives a mapping between the value-gradient estimator and its optimal weight. In \secref{exp-smallscale}, we compare our final recipe derived
by following the design-space map to its optimum, and show the comparison with previous baseline models across different reward settings. 

\vspace{-2pt}
\subsection{Settings}\label{sec:var-reduction-exp}
\vspace{-2pt}

\paragraph{Setup.}
We conduct our experiments on the SD-3.5 Medium model~\citep{esser2024scaling} and Qwen-Image~\citep{wu2025qwen} at a resolution of $512\times512$. Following Flow-GRPO~\citep{liu2025flow} and DiffusionNFT~\citep{zheng2025diffusionnft}, we set group size $G=24$, $T=10$ denoising steps for training and $T=40$ for evaluation. In addition, we choose to fine-tune the pretrained model using LoRA~\citep{hu2022lora} and disable classifier-free guidance (CFG)~\citep{ho2022classifier} in the training and evaluation stages for fair comparison. We evaluate our methods on various rewards including PickScore~\citep{kirstain2023pick}, visual-text OCR accuracy~\citep{chen2023textdiffuser}, and GenEval~\citep{ghosh2023geneval}. See \appxref{experimental-details} for details.

\paragraph{How the variance-reduction trick works.} In \secref{method-vr} and \secref{method-designspace}, we show the correctness of the one-sample estimator and view it as the estimator for the variance-reduced form of Term A in \eqnref{sde-estimator}. Although the methods exhibit strong theoretical properties, it's unclear whether it helps to improve the training convergence of RL algorithms. Previous works AWM~\citep{xue2025advantage} and DiffusionNFT~\citep{zheng2025diffusionnft} show that this type of estimator accelerates the training convergence of RL training under specific weight design (\tabref{master}), it's unclear either the estimator or the specific weight design leads to the empirical performance. To find out this, we directly compare the training performance of three training objectives: (1) the original objective \eqnref{sde-estimator} with $\eta_t=1$ (corresponding to the reverse of the forward process) ; (2) the one-sample estimator with inherited weighting \eqnref{temp-loss}; (3) A reweighted version of \eqnref{temp-loss} (same as AWM~\citep{xue2025advantage}).

The results are shown in \figref{estimator-bias-variance}(a). The one-sample estimator with inherited weighting suffers instability and ultimately fails, indicating that the variance reduction alone can not accelerate the training; the reweighting tricks are also indispensable. With the reweighted form, the one-sample estimator achieves remarkably faster convergence speed. This observation inspires us to investigate the weight design for each value-gradient estimator in the next subsection.

\paragraph{Direct bias--variance measurement.}
We directly measure the estimation error in the setting that matches the conditions of \propref{estimator-unbiased}. We set $\eta_t=1$ and evaluate SD3.5-Medium with PickScore on four prompts, using 128 trajectories per prompt and 40 denoising steps. As shown in \figref{estimator-bias-variance}(a), the stochastic estimator has $13.8\times$ the variance of the deterministic one-sample estimator, directly supporting the variance-reduction mechanism in \thmref{variance-reduction}. The matched KDE estimator ($h=1$) further reduces variance by $67.5\%$ relative to the deterministic estimator. Squared bias contributes less than $2\%$ of the MSE in every setting, so the observed error is dominated by variance.

\begin{figure}[t]
    \centering
    \begin{subfigure}[b]{0.3\textwidth}
        \centering
        \includegraphics[width=\linewidth]{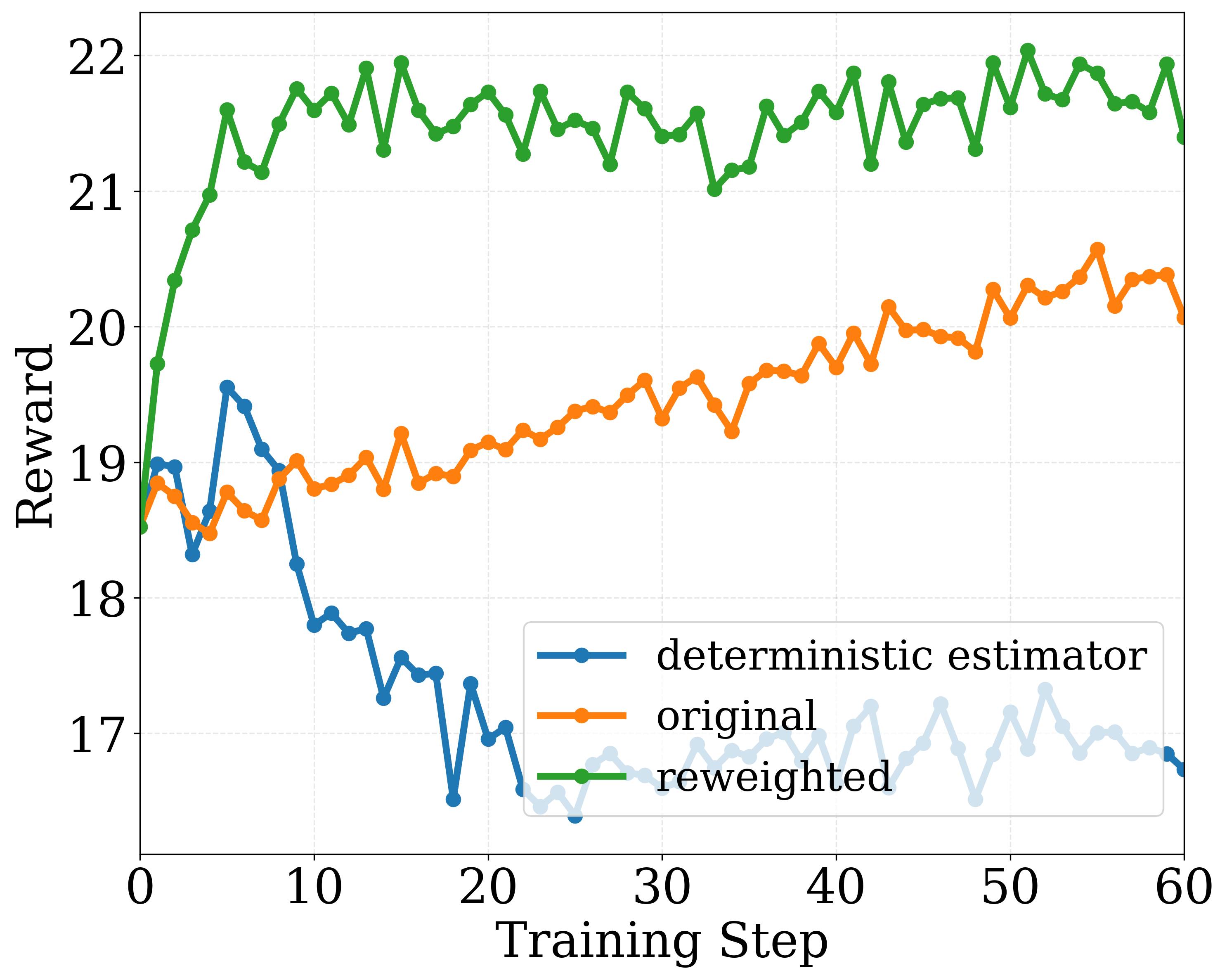}
        \caption{}
        \label{fig:var-reduction}
    \end{subfigure}
    \begin{subfigure}[b]{0.34\textwidth}
        \centering
        \includegraphics[width=\linewidth]{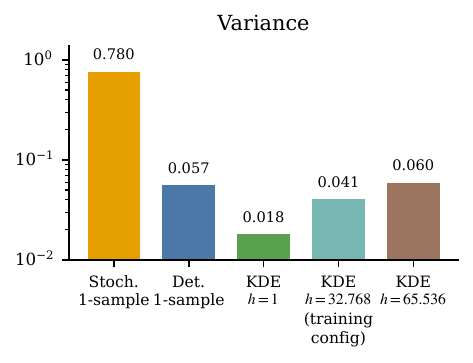}
        \caption{}
        \label{fig:estimator-variance}
    \end{subfigure}%
    \hfill
    \begin{subfigure}[b]{0.34\textwidth}
        \centering
        \includegraphics[width=\linewidth]{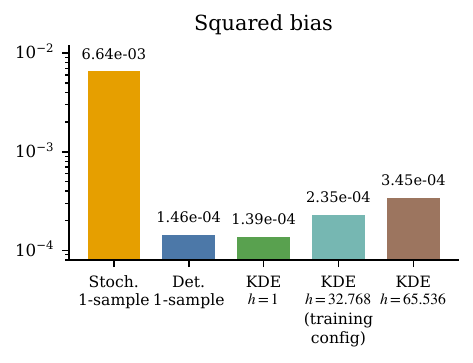}
        \caption{}
        \label{fig:estimator-squared-bias}
    \end{subfigure}
    \caption{\textbf{(a)} Comparison between different objectives on Pickscore reward. Original objective \eqnref{sde-estimator} with $\eta_t=1$ (Orange line); one-sample estimator with weighting in \eqnref{temp-loss} (Blue line); reweighted version of \eqnref{temp-loss} (Green line). \textbf{(b,c)} Empirical variance and squared bias of the value-gradient estimators. We set $\eta_t=1$ and evaluate SD3.5-Medium with PickScore reward on four prompts, using 128 trajectories per prompt and 40 denoising steps. At every intermediate state, the reference value gradient is approximated from all 128 trajectories; each estimator is computed with $G=24$ then compared with this common reference. The KDE training configuration uses $h=32.768$.
    }
    \label{fig:estimator-bias-variance}
    \vspace{-8pt}
\end{figure}

\paragraph{Bias--smoothness trade-off.}
The bandwidth controls a practical bias--smoothness trade-off. Although $h=1$ yields the lowest variance (\propref{estimator-unbiased} and \figref{estimator-bias-variance}(b)), its kernel weights can change abruptly across adjacent denoising steps in the high-dimensional latent space. Increasing $h$ smooths this temporal transition and improves optimization stability, at the cost of weakening posterior matching and increasing both variance and bias. We therefore use the intermediate $h$ that retains a meaningful variance reduction while producing smoother value-gradient targets along the trajectory.

\subsection{A Mapping between the Value-Gradient Estimator and Its Optimal Weight}\label{sec:mapping}
\paragraph{On-policy weighting $w_2$.}

\begin{figure}[htbp]
    \centering
    \begin{minipage}{0.94\textwidth}
        \centering
    
        \begin{subfigure}[b]{0.32\textwidth}
            \centering
            \includegraphics[width=\linewidth]{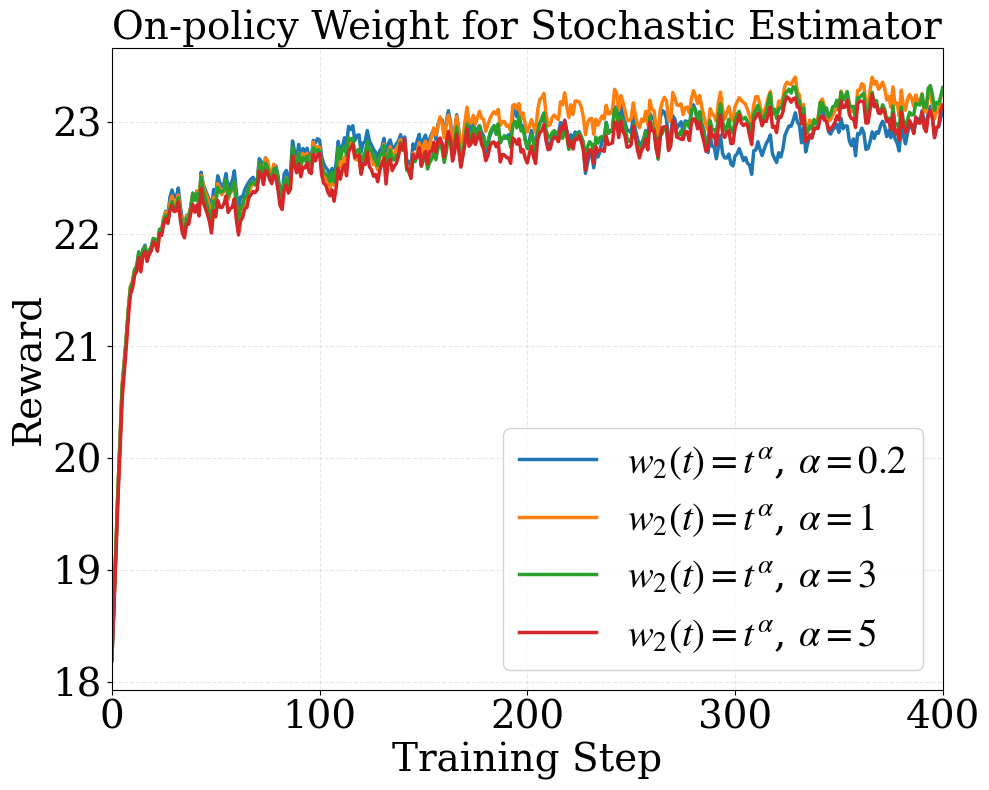}
            \caption{}
            \label{fig:1}
        \end{subfigure}%
        \hfill
        \begin{subfigure}[b]{0.32\textwidth}
            \centering
            \includegraphics[width=\linewidth]{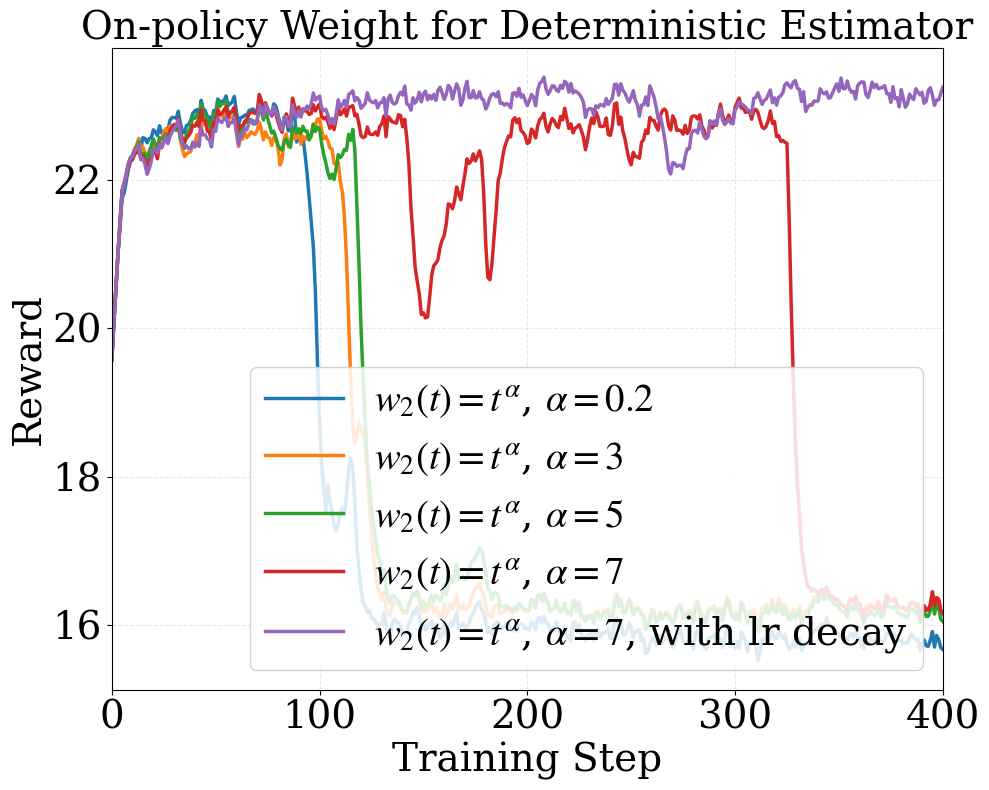}
            \caption{}
            \label{fig:2}
        \end{subfigure}%
        \hfill
        \begin{subfigure}[b]{0.32\textwidth}
            \centering
            \includegraphics[width=\linewidth]{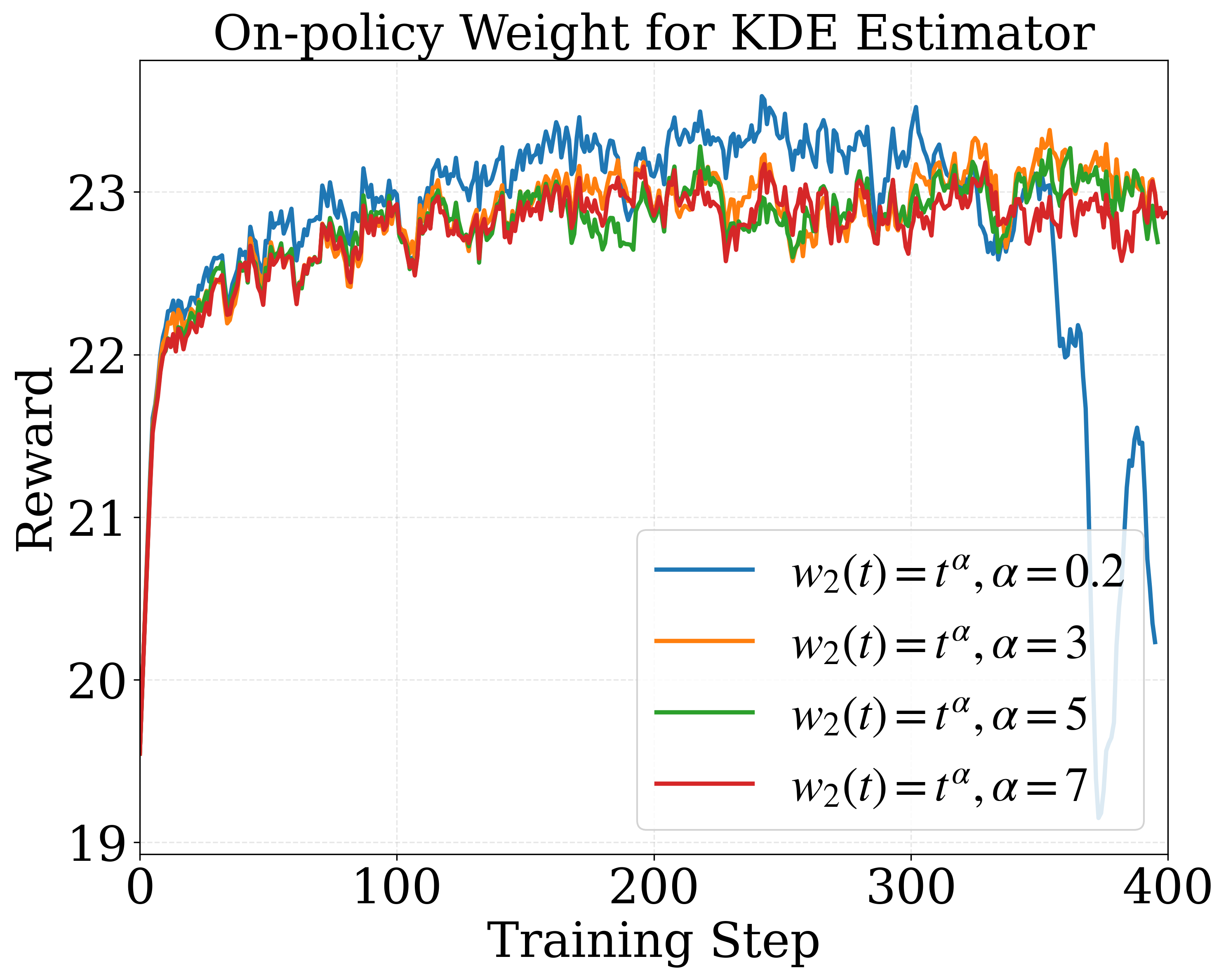}
            \caption{}
            \label{fig:3}
        \end{subfigure}
        
        \vspace{0.2em}
        
        \begin{subfigure}[b]{0.32\textwidth}
            \centering
            \includegraphics[width=\linewidth]{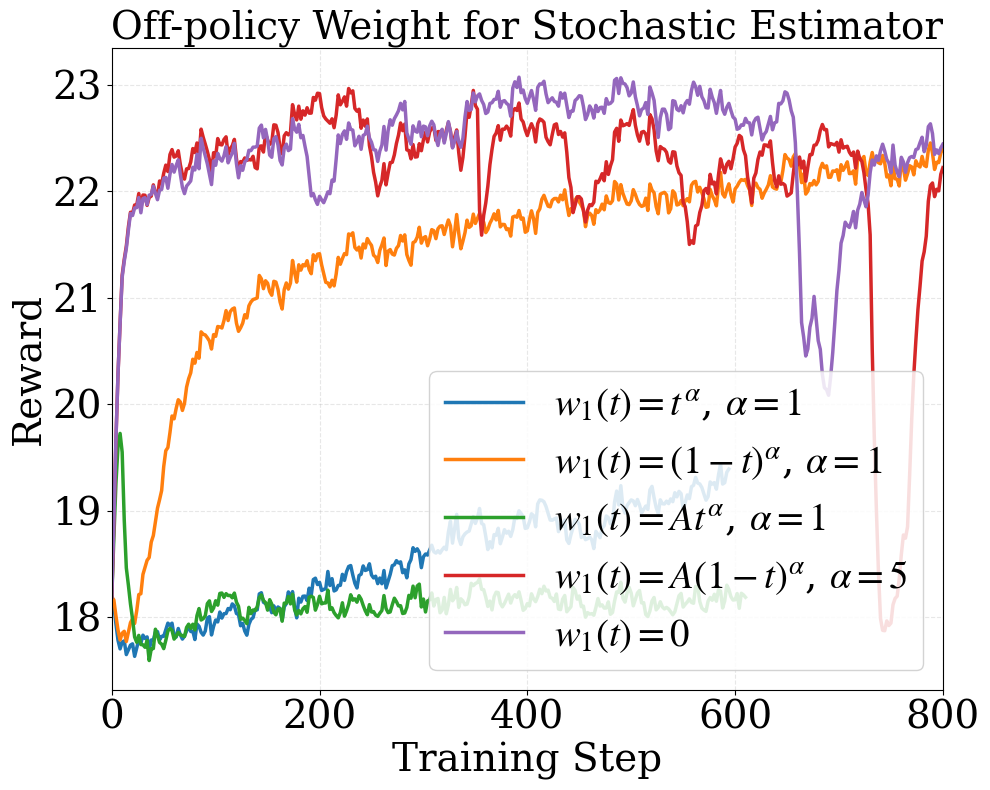}
            \caption{}
            \label{fig:4}
        \end{subfigure}%
        \hfill
        \begin{subfigure}[b]{0.32\textwidth}
            \centering
            \includegraphics[width=\linewidth]{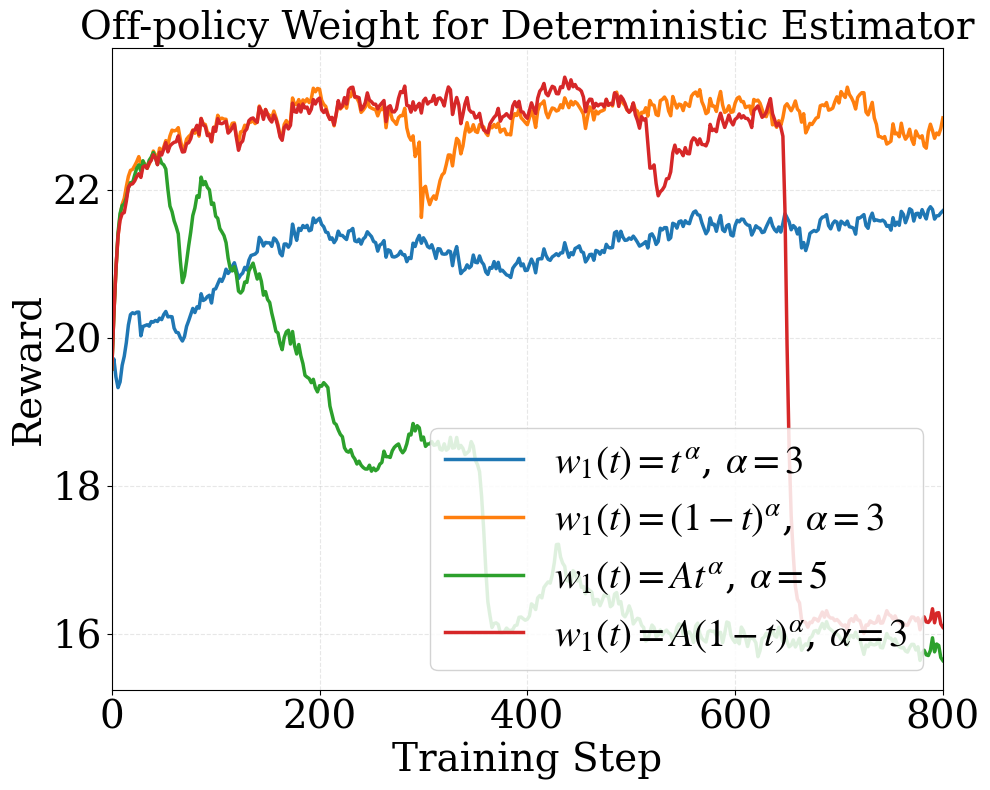}
            \caption{}
            \label{fig:5}
        \end{subfigure}%
        \hfill
        \begin{subfigure}[b]{0.32\textwidth}
            \centering
            \includegraphics[width=\linewidth]{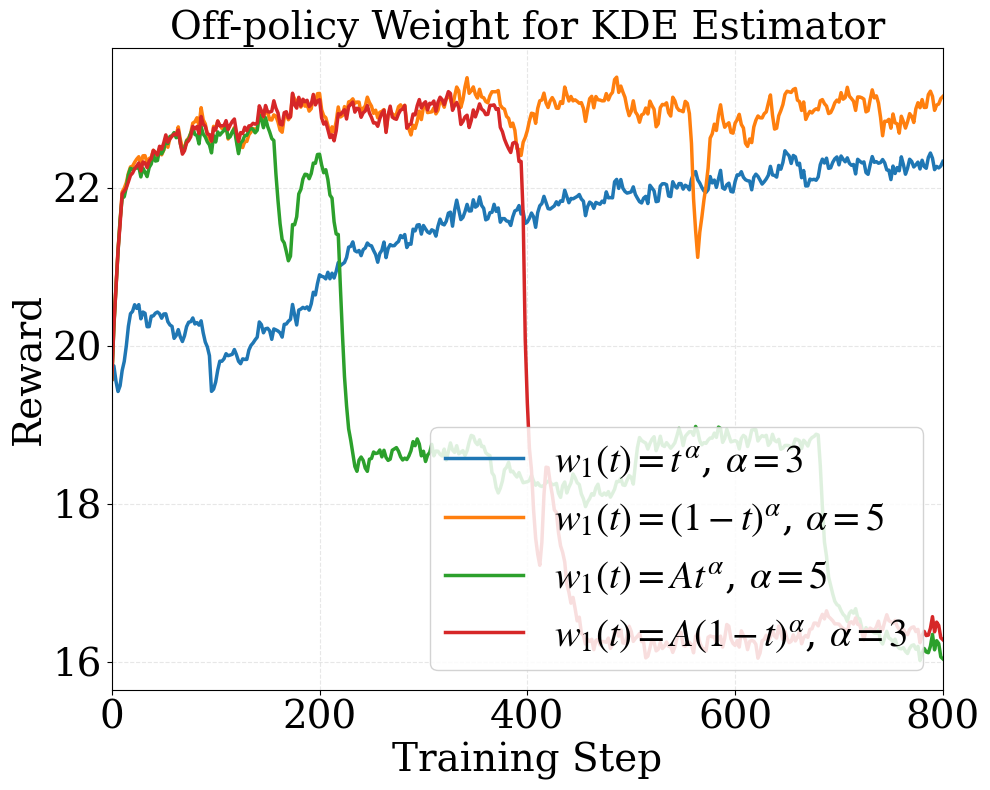}
            \caption{}
            \label{fig:6}
        \end{subfigure}
        
    \end{minipage}
    \caption{On-policy weight (row 1) and off-policy (row 2) weight ablations for the stochastic, deterministic, and our KDE estimator. The results for the same estimator are in the same column.}
    \label{fig:weight-ablation}
    \vspace{-8pt}
\end{figure}

We study three types of estimators: stochastic one-sample estimator $\widehat{\nabla V}_t^{\,\mathrm{sto}}$, deterministic one-sample estimator $\widehat{\nabla V}_t^{\,\mathrm{det}}$ and our proposed KDE estimator $\widehat{\nabla V}_t^{\,\mathrm{KDE}}$. As introduced in \secref{method-principled}, we investigate the weight: \(w_2(t)=t^\alpha\). We find that different sampler noise schedules affect the behavior (\appxref{add-results}), and we choose the best sampler noise schedule for weight ablation. The results are shown in \figref{weight-ablation} (first row). By adjusting the weights parameter $\alpha$ (see \figref{weight-visual} for visualization), we find that a different estimator exhibits a different behavior for weight preference. The stochastic estimators are not so sensitive w.r.t. the weight adjustment, $w_2=t$ performs slightly better (\figref{weight-ablation} (a)). Weight with larger $\alpha$ helps stabilize the training of the deterministic estimator, but can not fully resolve it without learning rate decay (\figref{weight-ablation} (b)). Our KDE estimator exhibits better training stability under weight with larger $\alpha$ without any extra design (\figref{weight-ablation} (c)). Larger $\alpha$ improves stability also implies that more training effort should be allocated to high-noise regions, which affects the sampling process at early stages.

\vspace{-2pt}
\paragraph{Off-policy weighting $w_1$.}
Based on the optimal on-policy weight, we adjust the off-policy weight $w_1$. As introduced in \secref{method-principled}, we adopt the following weight design: $w_1 = \{0, t^\alpha, A t^\alpha, (1-t)^\alpha, A(1-t)^\alpha\}$. For stochastic estimator, we find that $w_1=0$ works well (design of GRPO-Guard in \tabref{master}); the weight in $\{t^\alpha, (1-t)^\alpha\}$ overemphasize the regularization term and leads to slow convergence, while weights $\{At^\alpha, A(1-t)^\alpha\}$ leads to training instability (\figref{weight-ablation} (d)). Similarly, for the deterministic estimator and our KDE estimator, $t^\alpha$ degrades the training convergence speed and $\{At^\alpha, A(1-t)^\alpha\}$ causes training instability  (\figref{weight-ablation} (e,f)). We find that $w_1(t)=(1-t)^\alpha$ balances the weight between the off-policy regularization and the policy gradient in~\eqnref{unified-loss}, which implies that regularization in low-noise regions helps to improve the training stability.

\vspace{-3pt}
\subsection{Comparison with Prior Works}\label{sec:exp-smallscale}
\vspace{-3pt}
\begin{figure}[htbp]
    \centering

    \begin{minipage}{0.98\textwidth}
        \centering
    
        \begin{subfigure}[b]{0.24\textwidth}
            \centering
            \includegraphics[width=\linewidth]{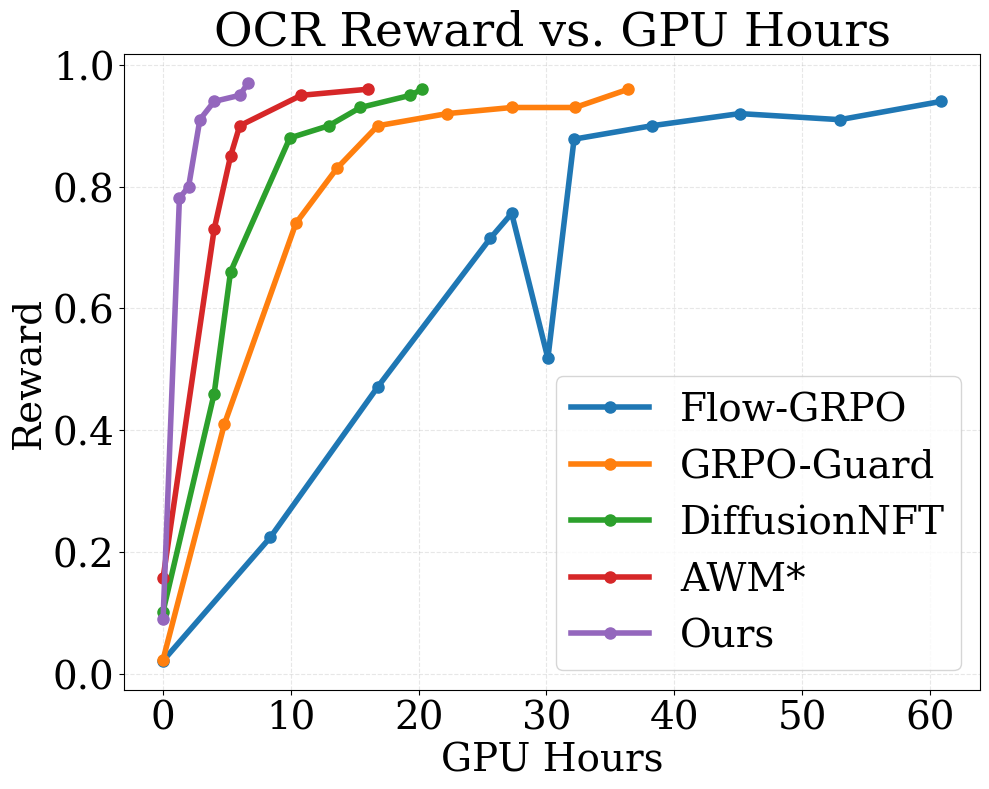}
            \caption{}
            \label{fig:1}
        \end{subfigure}%
        \hfill
        \begin{subfigure}[b]{0.24\textwidth}
            \centering
            \includegraphics[width=\linewidth]{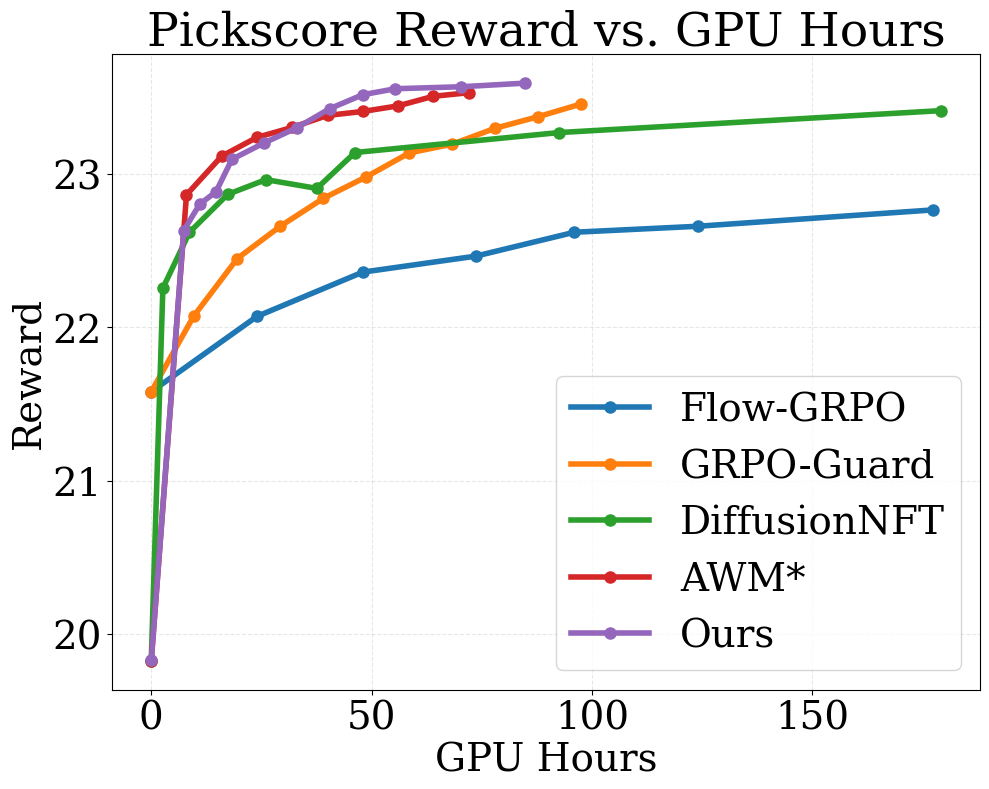}
            \caption{}
            \label{fig:2}
        \end{subfigure}%
        \hfill
        \begin{subfigure}[b]{0.24\textwidth}
            \centering
            \includegraphics[width=\linewidth]{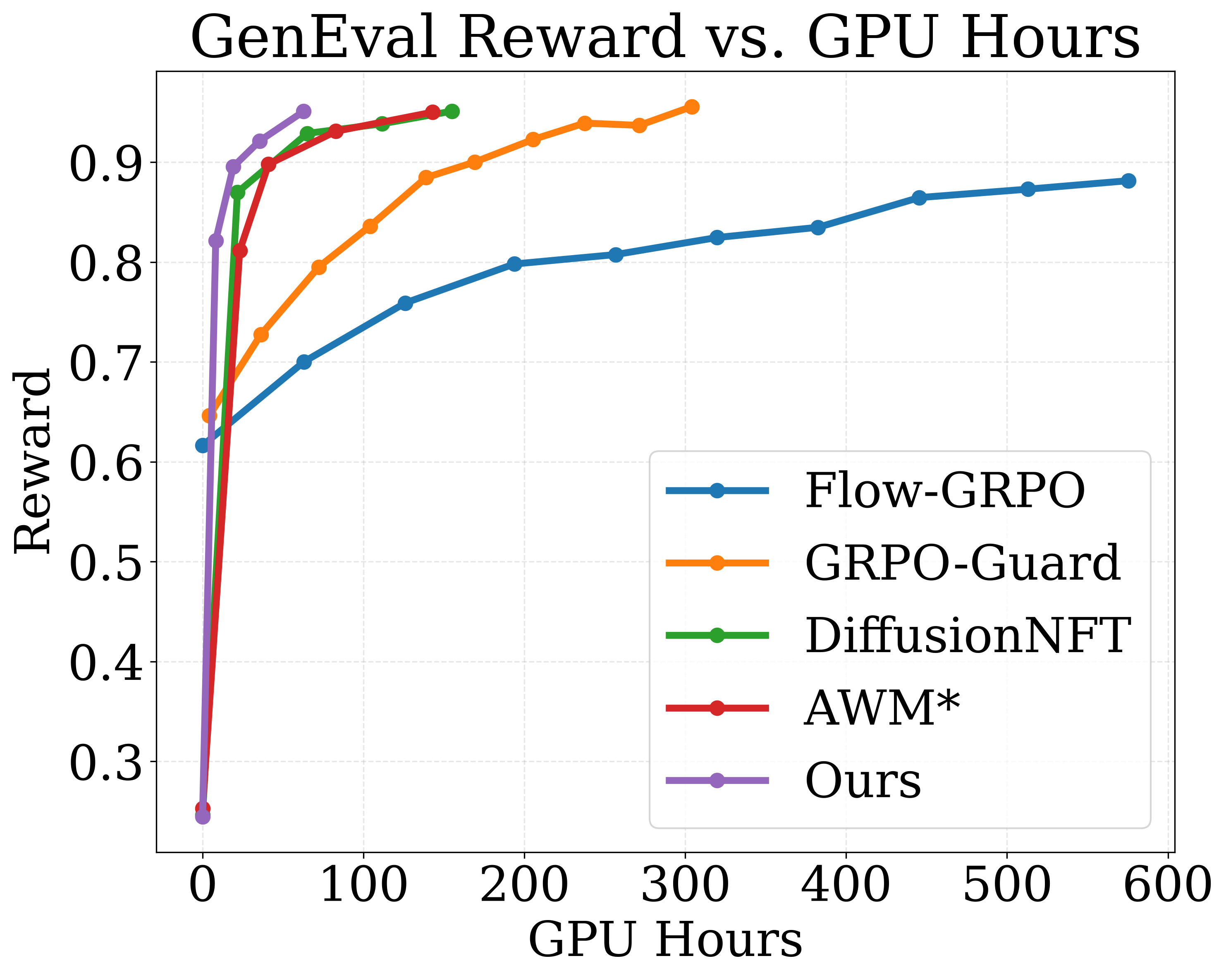}
            \caption{}
            \label{fig:3}
        \end{subfigure}    
        \hfill
        \begin{subfigure}[b]{0.24\textwidth}
            \centering
            \includegraphics[width=\linewidth]{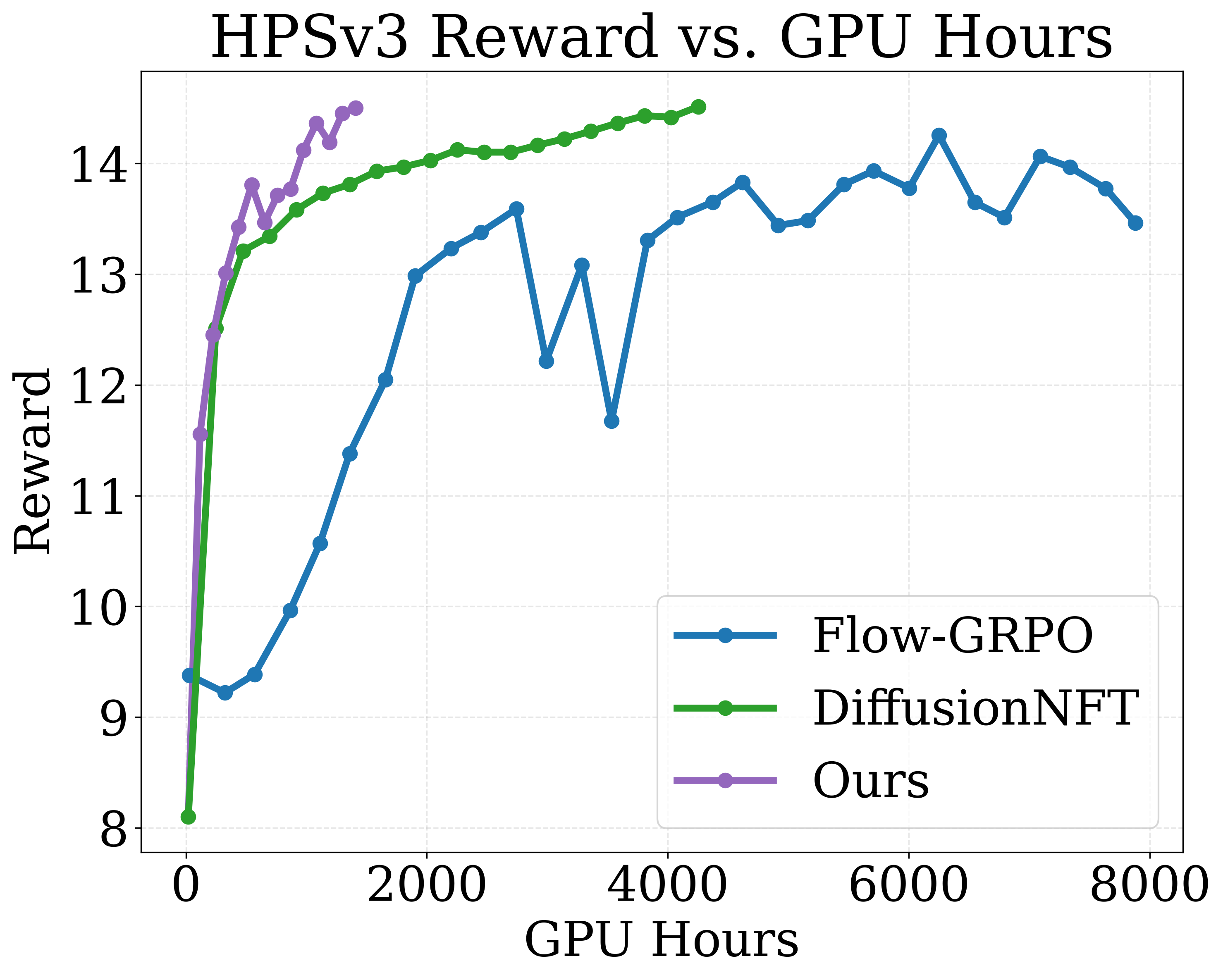}
            \caption{}
            \label{fig:3}
        \end{subfigure}        
    \end{minipage}
    \caption{\textbf{(a,b,c).} Training time comparison between our method and several baselines. \textbf{(d).} Comparison between Flow-GRPO, DiffusionNFT, and our method on Qwen-Image, HPSv3 reward. * means using the official implementation with different training configs compared to other baselines.}
    \label{fig:exp-hour}
    \vspace{-8pt}
\end{figure}
Based on the estimator--optimal weight mapping constructed in \secref{mapping}, we perform a training-time comparison between our best setting and previous strong baselines across several rewards. We use our KDE estimator with its optimal weighting $w_1=t^\alpha, w_2=(1-t)^\alpha$. The results are summarized in \figref{exp-hour} (a,b,c). On OCR reward, our method significantly accelerates training convergence compared to previous SOTA methods: 2$\times$ faster than AWM, $3\times$ faster than DiffusionNFT. On Pickscore and GenEval rewards, our method achieves comparable performance with AWM and DiffusionNFT without using other designs (e.g., specific EMA-KL loss in training, elaborate sampling design, training time-step selection), and outperforms all other baselines. We also conduct experiments on the large-scale model Qwen-Image~\citep{wu2025qwen} with HPSv3 reward~\citep{ma2025hpsv3}, and the result in \figref{exp-hour}(d) shows that our method achieves $4\times$ faster convergence speed than DiffusionNFT under the same configurations. 

\begin{wrapfigure}{r}{0.5\textwidth}
  \vspace{-8pt}
  \centering
  \includegraphics[width=0.95\linewidth]{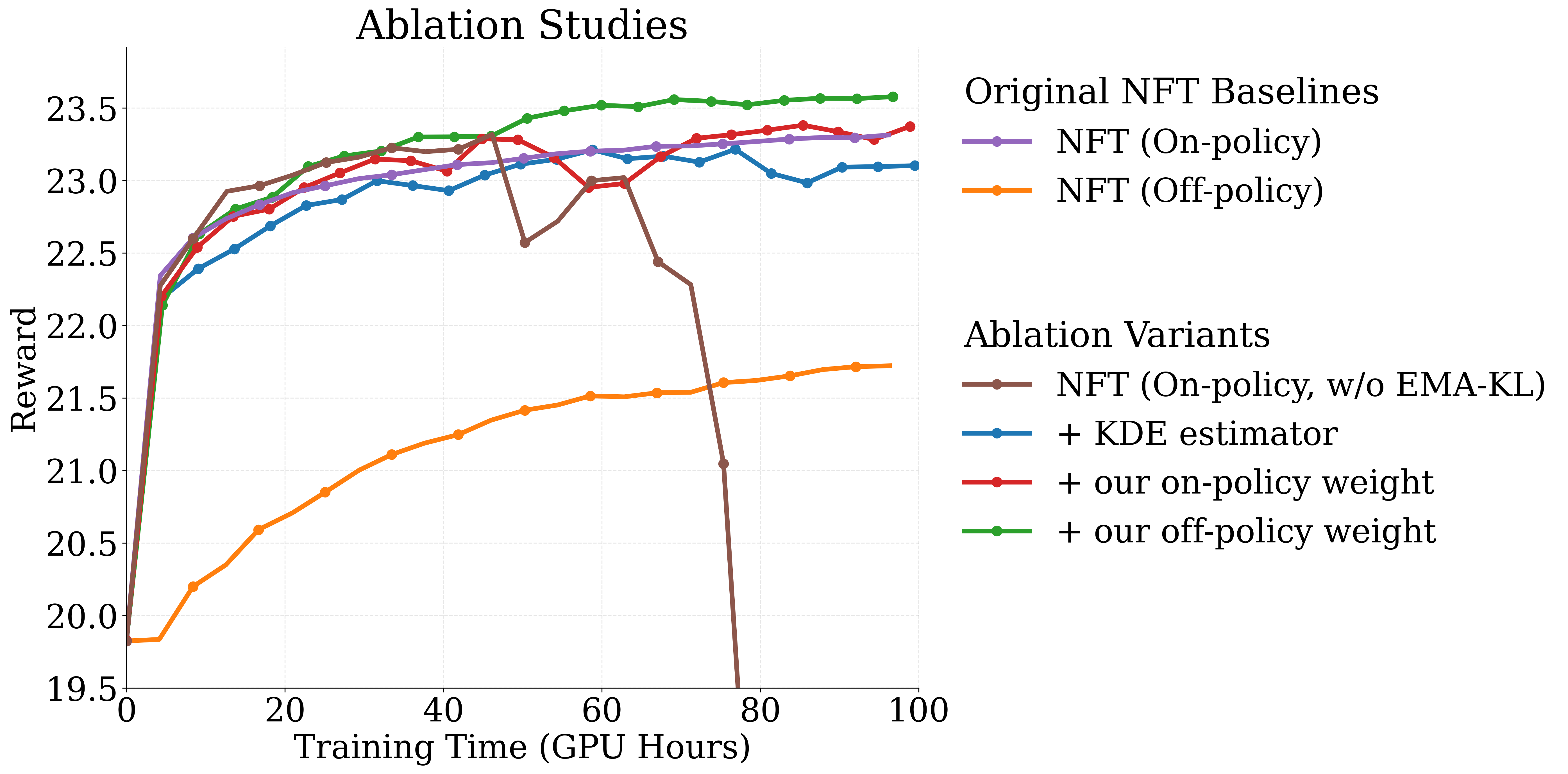}
  \caption{Ablation studies on the components of our final recipe.}
  \vspace{-16pt}
  \label{fig:ablation}
\end{wrapfigure}

\paragraph{Ablation study.}
To isolate the source of the gain, we start the ablation study from DiffusionNFT without its EMA-KL stabilization. The results are summarized in \figref{ablation}. The without EMA-KL baseline is unstable and eventually fails after 60 GPU hours. Replacing only its deterministic one-sample estimator with our KDE estimator improves its stability but hurts the performance; variance reduction alone is therefore insufficient. Using our on-policy weight helps stabilize and improve training performance. Finally, combining KDE with both our on-policy and off-policy weights reaches $23.52$ after 60 hours. For comparison, the original DiffusionNFT (on-policy, with EMA-KL) reaches $23.382$ after 160 hours, so our recipe is roughly $3\times$ faster.

\vspace{-2pt}
\section{Conclusion}\label{sec:conclusion}
\vspace{-2pt}
In this work, we develop a unified path-space framework that places previous diffusion-RL algorithms on a common theoretical foundation. Starting from the original RL objective, we derive a general training objective based on the variance-reduced form of the original objective. Previous reverse-trajectory methods and forward-matching methods can all be viewed as instances of this general objective. The derivation also yields a unified design space, including the value-gradient estimation, weight functions, and sampling choices. We propose a multi-sample KDE value gradient estimator and find a mapping between different estimators and their optimal weight design. Experiments on SD3.5-M and Qwen-Image models validate our theoretical explanation and show the empirical improvement of our result recipe.

\begin{ack}
DH is supported by National Science and Technology Major Project (2022ZD0114902), National Science Foundation of China (NSFC62376007), National Science Foundation of China (under Key Project No. 92570203), Beijing Natural Science Foundation (Z250001) and Beijing Major Science and Technology Project under Contract no. Z251100008425004. This work is supported in part by the State Key Laboratory of General Artificial Intelligence.
\end{ack}

\bibliography{main}
\bibliographystyle{plainnat}


\newpage
\appendix
\section{Organization of Appendix}
\label{app:organisation}

The organization of the appendix is as follows. \appxref{backgrounds} provides background in stochastic calculus. \appxref{proofs} gives all the proofs of the theorems in the main text. \appxref{table-proof}  gives the  derivations in \tabref{master}. \appxref{related} discusses the related works. \appxref{experimental-details} provides experimental details for \secref{exp} and shows some additional results. Finally, qualitative examples are shown in \appxref{qualitative}.

\section{Backgrounds for Stochastic Analysis}\label{appx:backgrounds}
In this section, we recall some basic results from stochastic calculus. Please see \citet{oksendal2013stochastic} or \citet{karatzas2014brownian} for details. Assume the probability space is $(\Omega,\clF,\bbP)$. A $d$-dimensional stochastic process is a parametrized collection of random vectors $\{\bfxx_t\}_{t\in T}$, where $T$ is the time interval. For each $\omega\in\Omega$ we can consider the function $t\to\bfxx_t(\omega), ~t\in T$, which is called the path of $\bfxx_t$. Assume $\bfxx$ is jointly measurable in $(t,\omega)$, then the stochastic process is a probability measure $\bbP$ on the measurable space $((\bbR^d)^T, \clB)$, where $(\bbR^d)^T$ contains all functions from $T$ to $\bbR^d$. In this paper, we only consider the stochastic process with continuous sample paths; then the stochastic process is a probability measure on the measurable space $(C(T, \bbR^d), \clB)$, where  $C(T, \bbR^d)$ contains all continuous functions from $T$ to $\bbR^d$. In this paper, we focus on the diffusion process, which can be defined by its generator:

\begin{definition}\label{def:generator}
The \emph{infinitesimal generator} of a stochastic process $(\bfxx_t)_t$ for a function $\phi(\bfxx)$ is
    \begin{equation}
        \clL_t \phi(\bfxx) = \lim_{s\to0^+}\frac{\bbE [\phi(\bfxx_{t+s})|\bfxx_{t}=\bfxx]-\phi(\bfxx)}{s},
    \end{equation}
     where $\phi$ is a smooth function. For an Itô process defined as the solution to the SDE $\dd \bfxx_t=\bfff(\bfxx_t,t) \ud t + g_t \ud \bfww_t$, the generator is
     \begin{equation}
         \clL_t = \sum_{i=1}^d \bfff^i(\bfxx,t) \partial_i + \frac{g_t^2}{2} \sum_{i=1}^d \partial_i^2, \quad\text{where}\quad \partial_i=\frac{\partial}{\partial x_i}.
     \end{equation}
\end{definition}

In stochastic calculus, It\^{o}'s lemma is a useful tool for calculating derivatives as the chain rule in calculus.
\begin{lemma}[It\^{o}'s lemma; \citep{oksendal2013stochastic}, Thm.~4.2.1]
    Assume $\bfxx_t$ is the solution of SDE $\dd \bfxx_t=\bfff(\bfxx_t,t)\dd t+g_t\dd \bfww_t$. Then for any function  $f\in C^2(\bbR^d\times \bbR)$, we have
    \begin{align}
        \dd f(\bfxx_t,t)=\left(\partial_tf(\bfxx_t,t)+\nabla_\bfxx f(\bfxx_t,t)\cdot \bfff(\bfxx_t,t)+\frac{g_t^2}{2}\Delta f(\bfxx_t,t)\right)\dd t + g_t\nabla_\bfxx f(\bfxx_t,t)\cdot\dd \bfww_t.
    \end{align}
\end{lemma}

For the calculation of the value function, the following Kolmogorov's backward equation is useful.
\begin{proposition}[Kolmogorov's Backward Equation; \citep{oksendal2013stochastic}, Thm.~8.1.1]\label{prop:backward-eqn}
    Let $V_t(\bfxx)=\bbE[R(\bfxx_T,c)|\bfxx_t=\bfxx]$ be the value function. Then $V$ is the solution of the following PDE
    \begin{align}
        \frac{\partial V}{\partial t}+\clA_tV=0,\quad V_T(\bfxx)=R(\bfxx,c),
    \end{align}
    where $\bfxx_t$ satisfies the SDE $\dd \bfxx_t=\bfff(\bfxx_t,t)\dd t+g_t\dd \bfww_t$ and $\clA_t$ is its generator:
    \begin{align}
        \clA f(\bfxx,t)=\bfff(\bfxx,t)\cdot\nabla f(\bfxx,t)+\frac{g_t^2}{2}\Delta f(\bfxx,t), \quad\forall f\in C^2(\bbR^d\times\bbR).
    \end{align} 
\end{proposition}
For the calculation of the importance sampling ratio between two path measures of diffusion processes, Girsanov's theorem is useful:
\begin{theorem}[Girsanov; \citep{oksendal2013stochastic}, Thm.~8.6.6]
\label{thm:girsanov}
Let $\bfxx_t,\bfyy_t$ solve the coupled SDEs on $\bbR^d$
\begin{align}
    \dd\bfxx_t = \bfuu(\bfxx_t,t)\dd t + g_t\dd\bfww_t,\quad
    \dd\bfyy_t = \bfvv(\bfyy_t,t)\dd t + g_t\dd\bfww_t,\nonumber
\end{align}
with the same initial condition and diffusion coefficient, and let $\bbP,\bbQ$ be their path measures. Assume $\bfuu,\bfvv$ are sufficiently smooth and satisfy Novikov's condition
$\bbE[\exp(\tfrac{1}{2}\int_0^T\lrVert{(\bfuu-\bfvv)/g_t}^2\dd t)]<\infty$. Then
\begin{align}
    \frac{\dd\bbP}{\dd\bbQ}=\exp\!\left(\int_0^T\tfrac{1}{g_t}(\bfuu-\bfvv)(\bfyy_t,t)\cdot\dd\bfww_t-\int_0^T\tfrac{1}{2g_t^2}\lrVert{(\bfuu-\bfvv)(\bfyy_t,t)}^2\dd t\right).\nonumber
\end{align}
\end{theorem}

\section{Proofs}
\label{appx:proofs}

\subsection{Setup: path-space importance sampling}
\label{appx:setup}
Let $\bfvv_\base,\bfvv_\theta,\bfvv_\refm:\bbR^d\times[0,1]\to\bbR^d$ be three smooth velocity fields, and write $\bbQ,\bbP_\theta,\bbP_\refm$ for the corresponding path measures on $C([0,1],\bbR^d)$:
\begin{align}
    &\bbP_\theta:~\dd\bfxx_t = \left[\bfvv_\theta(\bfxx_t,t) + \frac{\eta_t}{1-t}\big(\bfxx_t+(1-t)\bfvv_\theta(\bfxx_t,t)\big)\right]\dd t + \sqrt{\frac{2t\eta_t}{1-t}}\dd\bfww_t,\\
    &\bbQ:~\dd\bfxx_t^\base = \left[\bfvv_\base(\bfxx_t^\base,t) + \frac{\eta_t}{1-t}\big(\bfxx_t^\base+(1-t)\bfvv_\base(\bfxx_t^\base,t)\big)\right]\dd t + \sqrt{\frac{2t\eta_t}{1-t}}\dd\bfww_t,\\
    &\bbP_\refm:~\dd\bfxx_t^\refm = \left[\bfvv_\refm(\bfxx_t^\refm,t) + \frac{\eta_t}{1-t}\big(\bfxx_t^\refm+(1-t)\bfvv_\refm(\bfxx_t^\refm,t)\big)\right]\dd t + \sqrt{\frac{2t\eta_t}{1-t}}\dd\bfww_t,
\end{align}
where $\bfxx_1, \bfxx_1^\base, \bfxx_1^\refm\sim\clN(\bfzro,\bfI)$.

Then the original objective \eqnref{raw-obj} becomes $\max \bbE_c\clL(\theta,c)$, where 
\begin{align}
    \clL(\theta,c)=\bbE_{\bbP_\theta}[R(\bfxx_0,c)]-\beta\KL(\bbP_\theta\Vert\bbP_\refm).
\end{align}
If we want to use the importance sampling technique, the objective becomes
\begin{align}
     \clL(\theta,c)=\bbE_{\bbQ}\left[\frac{\dd \bbP_\theta}{\dd \bbQ}R(\bfxx_0,c)\right]-\beta\bbE_{\bbQ}\left[\frac{\dd \bbP_\theta}{\dd \bbQ}\log\left(\frac{\dd\bbP_\theta}{\dd\bbP_\refm}\right)\right],
\end{align}
where $\frac{\dd \bbP_\theta}{\dd \bbQ}$ is the Radon-Nikodym derivative of $\bbP_\theta$ with respect to $\bbQ$. By \thmref{girsanov}, 
\begin{align}\label{eqn:is-ratio-appx}
   M_1(\theta)=-\!\int_0^{1}\!\sqrt{2\wwt_{1-r}}\,\Delta\bfvv_\theta(\bfxx_{1-r},1-r)\cdot\dd\bfww_r\,-\,\int_0^{1}\!\wwt_{1-r}\,\lrVert{\Delta\bfvv_\theta(\bfxx_{1-r},1-r)}^2\dd r,
\end{align}
where $r=1-t$,  $\Delta\bfvv_\theta:=\bfvv_\theta-\bfvv_\base$ and $\wwt_t\;=\;\frac{(1+\eta_t)^2}{4\eta_t}\cdot\frac{1-t}{t}$. The results recover~\eqnref{is-policy-loss}.

\subsection{Proof of \propref{policy-grad}}\label{appx:proof-policy-grad}

\begin{proposition}[Stochastic-integral path-space estimator]\label{prop:appx-policy-grad}
 Assume that the clipping or a small step size keeps $|\exp(M_1(\theta^-))-1|<\epsilon$ and $\bbE_\bbQ[\Vert A\nabla_\theta M_1(\theta)\Vert]<M$ bounded, where $\theta^-$ denotes the stop-gradient version of the current parameters. Then $-\nabla_\theta\clL^{\text{policy}}(\theta,c)\to$
    \begin{align}
        \nabla_\theta\bbE_\bbQ\!\left[A(\bfxx_0,c)\int_0^{1}\!\sqrt{2\wwt_{1-r}}\,\Delta\bfvv_\theta(\bfxx_{1-r},1-r)\cdot\dd\bfww_r+A(\bfxx_0,c)\int_0^{1}\!\wwt_{1-r}\,\lrVert{\Delta\bfvv_\theta(\bfxx_{1-r},r)}^2\dd r\right],
    \end{align}
    as $\epsilon\to 0$.
\end{proposition}
\begin{proof}
    Consider the policy-gradient term $\clL^{\text{policy}}(\theta,c)=\bbE_\bbQ[\exp(M_1(\theta))A(\bfxx_{0},c)]$ of \eqnref{is-policy-loss}. 
    The gradient of the policy-gradient term is given by
    \begin{align}\label{eqn:app-taylor}
        \nabla_\theta\clL^{\text{policy}}(\theta,c) = \bbE_\bbQ[A(\bfxx_{0},c)\exp(M_1(\theta^-))\nabla_\theta M_1(\theta)].
    \end{align}
    Under the assumption $|\exp(M_1(\theta^-))-1|<\epsilon$,
    \begin{align}
        &\Vert\nabla_\theta\clL^{\text{policy}}(\theta,c)-\bbE_\bbQ[A(\bfxx_{0},c)\nabla_\theta M_1(\theta)]\Vert\\&=\lrVert{\bbE_\bbQ\lrbrack{(\exp(M_1(\theta^-))-1)A\nabla_\theta M_1(\theta)}}\leq \bbE_\bbQ\lrbrack{|\exp(M_1(\theta^-))-1|\Vert A\nabla_\theta M_1(\theta)\Vert}\\
        &\leq \epsilon M.
    \end{align}
     By \eqnref{is-ratio-appx},
    \begin{align}
        M_1(\theta)=-\!\int_0^{1}\!\sqrt{2\wwt_{1-r}}\,\Delta\bfvv_\theta(\bfxx_{1-r},1-r)\cdot\dd\bfww_r\,-\,\int_0^{1}\!\wwt_{1-r}\,\lrVert{\Delta\bfvv_\theta(\bfxx_{1-r},r)}^2\dd r.
    \end{align}
    Then we conclude that $-\nabla_\theta\clL^{\text{policy}}(\theta,c)\to -\bbE_\bbQ[A(\bfxx_{0},c)\nabla_\theta M_1(\theta)]=$
    \begin{align}
        \nabla_\theta\bbE_\bbQ\!\left[A(\bfxx_0,c)\int_0^{1}\!\sqrt{2\wwt_{1-r}}\,\Delta\bfvv_\theta(\bfxx_{1-r},1-r)\cdot\dd\bfww_r+A(\bfxx_0,c)\int_0^{1}\!\wwt_{1-r}\,\lrVert{\Delta\bfvv_\theta(\bfxx_{1-r},r)}^2\dd r\right],
    \end{align}
    as $\epsilon\to 0$.
\end{proof}

\subsection{Proof of \thmref{variance-reduction}}\label{appx:proof-vr}

\paragraph{Convention on time direction.}
Before starting the proof, we fix the time convention. The FM-time $t\in[0,1]$ used in the main text has $t=0$ at the clean sample and $t=1$ at pure noise; under this convention, the sampling process runs backward from $t=1$ to $t=0$. For the It\^o calculus below, we use a sampling time $r:=1-t\in[0,1]$ and denote $\bfyy_r:=\bfxx_{1-r}=\bfxx_t$ that runs forward during sampling, with $r=0$ at the initial noise $\bfyy_0=\bfxx_1$ and $r=1$ at the final sample $\bfyy_1=\bfxx_{0}$.

\begin{theorem}[Variance-reduced path-space estimator]\label{thm:appx-variance-reduction}
Let $q$ be the path measure of the sampling SDE \eqnref{bg-samplingsde} with base drift $\bfvv_\base$ and noise schedule $\eta_t>0$, and let $V_t(\bfxx,c):=\bbE_q[A(\bfxx_0,c)\mid\bfxx_t=\bfxx]$ be the associated value function. Additionally, assume that the SDE coefficients and $V_t$ are smooth.
Under the trust-region condition of \propref{policy-grad}, the policy gradient of \eqnref{is-policy-loss} admits the (locally) variance-reduced representation
\begin{align}\label{eqn:appx-vr-form}
\resizebox{0.93\textwidth}{!}{$
\displaystyle
-\clL^{\text{policy}}(\theta,c)= {}_\theta\bbE_q\left[\int_0^1\left((1+\eta_t)\Delta\bfvv_\theta(\bfxx_t,t)\cdot\nabla_{\bfxx_t}V_t(\bfxx_t,c)+A(\bfxx_0,c)\wwt_t\lrVert{\Delta\bfvv_\theta(\bfxx_t,t)}^2\right)\dd t\right].
$}
\end{align}
Let $\widehat{\nabla V}_{t,\Delta t}^{\,\mathrm{sto}}:=A(\bfxx_0,c)\sqrt{(1-t)/(2\eta_t t\Delta t)}\,\bfeps_t$ be the local discretization of Term (A) in \eqnref{sde-estimator}. Assume $\widehat{\nabla V}_{t,\Delta t}^{\,\mathrm{sto}}$ is square-integrable for any interior time $t\in(0,1)$. When $\Delta t\downarrow0$, we have
\begin{align}
\mathrm{Var}_{\bbQ}\!\left(
\widehat{\nabla V}_{t,\Delta t}^{\,\mathrm{sto}}\right)
\geq
\mathrm{Var}_{\bbQ}\!\left(
\nabla_{\bfxx_t}V_t(\bfxx_t,c)\right).
\label{eqn:app-local-rb-var}
\end{align}
\end{theorem}
\begin{proof}
We prove the identity for the policy-gradient term; the quadratic Term~(B) in \eqnref{sde-estimator} is carried through unchanged. To avoid mixing the two time directions, we use the forward sampling time process $\bfyy_r:=\bfxx_{1-r}$ for $r\in[0,1]$. 
Then the sampling SDE (\eqnref{bg-samplingsde}) becomes
\begin{align}
    \bbQ:\quad\dd\bfyy_r = \underbrace{-\left[\bfvv_\base(\bfxx_r,1-r) + \frac{\eta_{1-r}}{r}\big(\bfyy_r+r\bfvv_\base(\bfyy_r,1-r)\big)\right]}_{:=\bfmu_r(\bfyy_r)}\dd r + \underbrace{\sqrt{\frac{2(1-r)\eta_{1-r}}{r}}}_{:=g_r}\,\dd\bfww_r,
\end{align}
where $\bfyy_0\sim\clN(\bfzro,\bfI)$. Under $\bbQ$, we simply denote the sampling process as
\begin{align}
\dd\bfyy_r=\bfmu_r(\bfyy_r)\dd r+g_{r}\dd\bfww_r.
\label{eqn:app-sampling-time-sde}
\end{align}
All calculations can first be performed on $[\varepsilon,1-\varepsilon]$ and then extended to $[0,1]$ by square-integrability and a limiting argument. Let
\begin{align}
H_r&:=\sqrt{2\wwt_{1-r}}\,\Delta\bfvv_\theta(\bfyy_r,1-r),
\label{eqn:app-H}\\
I_r&:=\int_0^r H_s\cdot\dd\bfww_s, \label{eqn:app-I}\\
\bar V_r(\bfxx,c) &:=\bbE_\bbQ[A(\bfyy_1,c)\mid\bfyy_r=\bfxx],\qquad
V_r:=\bar V_r(\bfyy_r,c).
\label{eqn:app-V}
\end{align}
By square-integrability, $I_r$ is an It\^o martingale. The process $V_r=\bbE_\bbQ[A(\bfyy_1,c)\mid\clF_r]$ is also a martingale by the tower property, with $V_1=A(\bfyy_1,c)=A(\bfxx_0,c)$. The Kolmogorov backward equation for $\bar V_r$ under \eqnref{app-sampling-time-sde} is
\begin{align}
\partial_r\bar V_r+\bfmu_r\cdot\nabla_{\bfxx}\bar V_r
+\frac{g_{r}^2}{2}\Delta_{\bfxx}\bar V_r=0.
\label{eqn:app-backward}
\end{align}
Applying It\^o's formula and using \eqnref{app-backward} eliminates the drift term:
\begin{align}
\dd V_r&=\partial_r\bar{V}_r(\bfyy_r,c)\dd r+\nabla_\bfxx\bar{V}_r(\bfyy_r,c)\cdot\dd\bfyy_r+\frac{1}{2}\Delta\bar{V}_r(\bfyy_r,c) \dd [\bfyy_r,\bfyy_r],\\
&=\partial_r\bar{V}_r(\bfyy_r,c)\dd r+\nabla_\bfxx\bar{V}_r(\bfyy_r,c)\cdot(\bfmu_r(\bfyy_r)\dd r+g_{r}\dd\bfww_r)+\frac{g_r^2}{2}\Delta\bar{V}_r(\bfyy_r,c) \dd r,\\
&=\underbrace{\lrparen{\partial_r\bar{V}_r(\bfyy_r,c) +\nabla_\bfxx\bar{V}_r(\bfyy_r,c)\cdot\bfmu_r(\bfyy_r) +\frac{g_r^2}{2}\Delta\bar{V}_r(\bfyy_r,c)}}_{=0,\,\text{by \eqnref{app-backward}}}\dd r+g_{r}\nabla_\bfxx\bar{V}_r(\bfyy_r,c)\cdot\dd\bfww_r\\
&=g_{r}\nabla_{\bfxx}\bar V_r(\bfyy_r,c)\cdot\dd\bfww_r,
\label{eqn:app-dV}
\end{align}
where $\Delta=\sum_{i=1}^d\frac{\partial^2}{\partial x_i^2}$ is the Laplace operator and $[\cdot,\cdot]$ is the quadratic variation of the diffusion process.
It remains to identify the cross-variation of the two martingales. \eqnref{app-H} and \eqnref{app-dV} give
\begin{align}
\dd[V,I]_r
&=g_{r}\sqrt{2\wwt_{1-r}}\,
\Delta\bfvv_\theta(\bfyy_r,1-r)\cdot
\nabla_{\bfxx}\bar V_r(\bfyy_r,c)\,\dd r\\
&=(1+\eta_{1-r})\,
\Delta\bfvv_\theta(\bfyy_r,1-r)\cdot
\nabla_{\bfxx}\bar V_r(\bfyy_r,c)\,\dd r,
\label{eqn:app-covariation}
\end{align}
where the second equality follows from
\begin{align}
g_{r}\sqrt{2\wwt_{1-r}}
=\sqrt{\frac{2t\eta_t}{1-t}}\,
\sqrt{\frac{(1+\eta_t)^2(1-t)}{2\eta_t t}}
=1+\eta_t,\qquad t=1-r.
\end{align}
It\^o's product rule now yields
\begin{align}
\dd(V_rI_r)
&=V_r\dd I_r+I_r\dd V_r+\dd[V,I]_r\\
&=\left(V_rH_r+I_rg_{r}\nabla_{\bfxx}\bar V_r(\bfyy_r,c)\right)\cdot\dd\bfww_r
 + (1+\eta_{1-r})\Delta\bfvv_\theta(\bfyy_r,1-r)\cdot
\nabla_{\bfxx}\bar V_r(\bfyy_r,c)\,\dd r.
\end{align}
After integration, 
\begin{align}
V_1I_1&=V_0I_0+\int_0^1\left[V_rH_r+g_{r}I_r\nabla_{\bfxx}\bar V_r(\bfyy_r,c)\right]\cdot\dd\bfww_r
 \\&\quad + \int_0^1(1+\eta_{1-r})\Delta\bfvv_\theta(\bfyy_r,1-r)\cdot
\nabla_{\bfxx}\bar V_r(\bfyy_r,c)\,\dd r.
\end{align}
Take expectation on both sides and by $I_0=0$, we get
\begin{align}
\bbE_\bbQ[V_1I_1]&=\bbE_\bbQ\!\left[
A(\bfxx_0,c)\int_0^1\sqrt{2\wwt_{1-r}}\,
\Delta\bfvv_\theta(\bfyy_{r},1-r)\cdot\dd\bfww_r\right]
\nonumber\\
\qquad&=\bbE_\bbQ\!\left[
\int_0^1(1+\eta_{1-r})\Delta\bfvv_\theta(\bfyy_{r},1-r)\cdot
\nabla \bar{V}_{r}(\bfyy_{r},c)\,\dd r\right]\\
&\quad+\underbrace{\bbE_\bbQ\left[\int_0^1\left(V_rH_r+g_{r}I_r\nabla_{\bfxx}\bar V_r(\bfyy_r,c)\right)\cdot\dd\bfww_r\right]}_{=0,\quad\text{by the martingale property of It\^o integral}}\\
&=\bbE_\bbQ\!\left[
\int_0^1(1+\eta_{1-r})\Delta\bfvv_\theta(\bfyy_{r},1-r)\cdot
\nabla \bar{V}_{r}(\bfyy_{r},c)\,\dd r\right].
\label{eqn:app-main-identity}
\end{align}
Changing variables back to FM-time $t=1-r$ turns the right-hand side into the first term of \eqnref{appx-vr-form}. Restoring the unchanged quadratic Term~(B) gives
\begin{align}
\resizebox{0.93\textwidth}{!}{$
\displaystyle
-\clL^{\text{policy}}(\theta,c)= {}_\theta\bbE_q\left[\int_0^1\left((1+\eta_t)\Delta\bfvv_\theta(\bfxx_t,t)\cdot\nabla_{\bfxx_t}V_t(\bfxx_t,c)+A(\bfxx_0,c)\wwt_t\lrVert{\Delta\bfvv_\theta(\bfxx_t,t)}^2\right)\dd t\right].
$}
\end{align}
which is exactly \eqnref{appx-vr-form}. 

We now prove \eqnref{app-local-rb-var}. For a positive sampling-time step $\Delta r$, 
\begin{align}
\widehat{\nabla V}_{t,\Delta r}^{\,\mathrm{sto}}
=\frac{A(\bfyy_1,c)}{g_r\Delta r}
\big(\bfww_{r+\Delta r}-\bfww_r\big),\qquad t=1-r.
\end{align}
By the tower property, the martingale identity
$V_{r+\Delta r}=\bbE_\bbQ[A(\bfyy_1,c)\mid\clF_{r+\Delta r}]$, and \eqnref{app-dV},
\begin{align}
\bbE_\bbQ\!\left[
\widehat{\nabla V}_{t,\Delta r}^{\,\mathrm{sto}}
\mid\clF_r\right]
&=\frac{1}{g_r\Delta r}
\bbE_\bbQ\!\left[
(V_{r+\Delta r}-V_r)
(\bfww_{r+\Delta r}-\bfww_r)
\mid\clF_r\right]\\
&=\frac{1}{g_r\Delta r}
\bbE_\bbQ\!\left[
\int_r^{r+\Delta r}
g_s\nabla_{\bfxx}\bar V_s(\bfyy_s,c)\,\dd s
\middle|\clF_r\right].
\label{eqn:app-local-rb-proof}
\end{align}
By Jensen's inequality and Fubini's theorem,
\begin{align}
\left\|
\bbE_\bbQ\!\left[
\widehat{\nabla V}_{t,\Delta r}^{\,\mathrm{sto}}\mid\clF_r\right]
-\nabla_{\bfxx}\bar V_r(\bfyy_r,c)
\right\|_{L^2(\bbQ)}
&\leq
\frac{1}{g_r\Delta r}
\int_r^{r+\Delta r}
\left\|
g_s\nabla_{\bfxx}\bar V_s(\bfyy_s,c)
-g_r\nabla_{\bfxx}\bar V_r(\bfyy_r,c)
\right\|_{L^2(\bbQ)}
\,\dd s\\
&\xrightarrow[\Delta r\downarrow0]{}0.
\end{align}
By the Markov property,
$\bbE_\bbQ[\widehat{\nabla V}_{t,\Delta t}^{\,\mathrm{sto}}\mid\clF_r]
=\bbE_q[\widehat{\nabla V}_{t,\Delta t}^{\,\mathrm{sto}}\mid\bfxx_t]$.
By the property of conditional variance, 
\begin{align}
\mathrm{Var}_{\bbQ}\!\left(
\widehat{\nabla V}_{t,\Delta t}^{\,\mathrm{sto}}\right)
&=
\mathrm{Var}_{\bbQ}\!\left(
\bbE_{\bbQ}[
\widehat{\nabla V}_{t,\Delta t}^{\,\mathrm{sto}}\mid\bfxx_t]\right)
+\bbE_{q}\!\left[
\mathrm{Var}_{q}\!\left(
\widehat{\nabla V}_{t,\Delta t}^{\,\mathrm{sto}}\mid\bfxx_t\right)\right]\\
&\geq
\mathrm{Var}_{\bbQ}\!\left(
\bbE_{\bbQ}[
\widehat{\nabla V}_{t,\Delta t}^{\,\mathrm{sto}}\mid\bfxx_t]\right)\xrightarrow[\Delta t\downarrow0]{}\mathrm{Var}_{\bbQ}\!\left(
\nabla_{\bfxx_t}V_t(\bfxx_t,c)\right).
\end{align}
Identifying the positive step magnitudes $\Delta r=\Delta t$ proves \eqnref{app-local-rb-var}.
\end{proof}

\subsection{Proof of \propref{single-point}}\label{appx:proof-single-point}

\begin{proposition}\label{prop:appx-single-point}
The gradient of the value function admits the closed-form expression
\begin{align}
    \nabla_{\bfxx_t} V_t(\bfxx_t, c)=\bbE_{q(\bfxx_0\mid\bfxx_t)}\!\left[\nabla_{\bfxx_t} \log q(\bfxx_t\mid\bfxx_0)\,A(\bfxx_0,c)\right] - V_t(\bfxx_t,c)\,\bbE_{q(\bfxx_0\mid\bfxx_t)}\!\left[\nabla_{\bfxx_t} \log q(\bfxx_t\mid\bfxx_0)\right].
\end{align}
Let $\bfvv(\bfxx_t,\bfxx_0):=(\bfxx_t-\bfxx_0)/t$ denote the FM conditional velocity. Under $q$ with $\bfvv_\base=\bbE[\bfvv\mid\bfxx_t]\approx\bfvv_\old$ and $\eta_t\equiv 1$ (i.e.,~$q$ is the reverse of the forward noising process), this specialises to
\begin{align}\label{eqn:appx-gradV-centered}
    \nabla_{\bfxx_t} V_t(\bfxx_t,c)=-\tfrac{1-t}{t}\,\bbE_{q(\bfxx_0\mid\bfxx_t)}\!\left[A(\bfxx_0,c)\,\big(\bfvv(\bfxx_t,\bfxx_0)-\bfvv_\base(\bfxx_t,t)\big)\right].
\end{align}
The corresponding one-sample estimator on a rollout trajectory $\bfxx_{0:1}$ is 
\begin{align}
\widehat{\nabla V}_t^{\,\mathrm{det}}(\bfxx_t,c):=-\tfrac{1-t}{t}A(\bfxx_0,c)\big(\bfvv(\bfxx_t,\bfxx_0)-\bfvv_\old(\bfxx_t,t)\big).
\end{align}
\end{proposition}

\begin{proof}
    Under $\bbQ$ with $\eta_t\!\equiv\!1$, the sampling SDE reduces to the reverse of the forward noising process (\secref{background-fm}), and the forward transition kernel is $q(\bfxx_t\mid\bfxx_0)=\clN(\bfxx_t;(1-t)\bfxx_0,t^2\bfI)$. Differentiating the value function, we get
\begin{align}
    \nabla_{\bfxx_t} V_t(\bfxx_t)&=\nabla_{\bfxx_t}\frac{\bbE_{\bfxx_0}[R(\bfxx_0,c)q(\bfxx_t|\bfxx_0)]}{\bbE_{\bfxx_0}[q(\bfxx_t|\bfxx_0)]}\\
    &=\frac{\bbE_{\bfxx_0}\lrbrack{\nabla_{\bfxx_t} q(\bfxx_t|\bfxx_0)R(\bfxx_0,c)}}{\bbE_{\bfxx_0}\lrbrack{q(\bfxx_t|\bfxx_0)}}-\frac{\bbE_{\bfxx_0}\lrbrack{ q(\bfxx_t|\bfxx_0)R(\bfxx_0,c)}~\bbE_{\bfxx_0}\lrbrack{ \nabla_{\bfxx_t} q(\bfxx_t|\bfxx_0)}}{\lrparen{\bbE_{\bfxx_0}\lrbrack{q(\bfxx_t|\bfxx_0)}}^2}
    \\
    &=\frac{\bbE_{\bfxx_0}\lrbrack{q(\bfxx_t|\bfxx_0)\nabla_{\bfxx_t} \log q(\bfxx_t|\bfxx_0)R(\bfxx_0,c)}}{\bbE_{\bfxx_0}\lrbrack{q(\bfxx_t|\bfxx_0)}} \\&\quad -\frac{\bbE_{\bfxx_0}\lrbrack{ q(\bfxx_t|\bfxx_0)R(\bfxx_0,c)}~\bbE_{\bfxx_0}\lrbrack{  q(\bfxx_t|\bfxx_0)\nabla_{\bfxx_t} \log q(\bfxx_t|\bfxx_0)}}{\lrparen{\bbE_{\bfxx_0}\lrbrack{q(\bfxx_t|\bfxx_0)}}^2}\\
    &=\bbE_{q(\bfxx_0|\bfxx_t)}\left[\nabla_{\bfxx_t} \log q(\bfxx_t|\bfxx_0)R(\bfxx_0,c)\right] - V_t(\bfxx_t) \bbE_{q(\bfxx_0|\bfxx_t)}\lrbrack{\nabla_{\bfxx_t} \log q(\bfxx_t|\bfxx_0)}.\label{eqn:app-bayes}
\end{align}
The Gaussian score of $q(\bfxx_t|\bfxx_0)$ is
\begin{align}\label{eqn:app-gaussian-score}
    \nabla_\bfxx\log q(\bfxx_t\mid\bfxx_0) = -\frac{\bfxx_t-(1-t)\bfxx_0}{t^2} = -\frac{1-t}{t}\cdot\bfvv(\bfxx_t,\bfxx_0)-\frac{1}{t}\bfxx_t,
\end{align}
where $\bfvv(\bfxx_t,\bfxx_0):=(\bfxx_t-\bfxx_0)/t$ is the FM conditional velocity. Substituting into \eqnref{app-bayes} and replacing $R$ with the group-relative advantage $A$ (it can be viewed as a normalized reward function),
\begin{align}
    \nabla_\bfxx V_t(\bfxx_t,c) = -\tfrac{1-t}{t}\,\bbE_{q(\bfxx_0\mid\bfxx_t)}\!\left[A(\bfxx_0,c)(\bfvv(\bfxx_t,\bfxx_0)-\bbE[\bfvv(\bfxx_t,\bfxx_0)|\bfxx_t]\right].\label{eqn:app-v-simple}
\end{align}
Any good approximation $\bfvv_\old(\bfxx_t,t)\approx\bbE[\bfvv|\bfxx_t]$ (which holds whenever $\bfvv_\old$ is a well-trained diffusion model) gives
\begin{align}
    \nabla_\bfxx V_t(\bfxx_t,c) = -\tfrac{1-t}{t}\,\bbE_{q(\bfxx_0\mid\bfxx_t)}\!\left[A(\bfxx_0,c)\,(\bfvv(\bfxx_t,\bfxx_0)-\bfvv_\old(\bfxx_t,t))\right],\label{eqn:app-v-centered}
\end{align}
whose one-sample Monte Carlo estimator is $\widehat{\nabla V}_t=-\tfrac{1-t}{t}A(\bfvv(\bfxx_t,\bfxx_0)-\bfvv_\old(\bfxx_t,t))$.
\end{proof}

\subsection{Proof of \propref{estimator-unbiased}}
\label{appx:proof-estimator-unbiased}
\begin{proposition}[Properties of the value-gradient estimators]\label{prop:appx-estimator-unbiased}
Let $A(\bfxx_0,c)$ be square-integrable. Assume $\bfvv_\base=\bfvv_\old$, then the following arguments hold:
\begin{enumerate}
    \item[(i).] The discretized stochastic one-sample estimator is asymptotically unbiased in the continuous-time It\^o limit of \thmref{variance-reduction}: $\lim_{\Delta t\to0}\bbE\!\left[\widehat{\nabla V}_{t,\Delta t}^{\,\mathrm{sto}}\mid\bfxx_t\right]=\nabla_{\bfxx_t}V_t(\bfxx_t,c)$;
    \item[(ii).] Assume $\eta_t\equiv1$, $\bfvv_\old(\bfxx_t,t)=\bbE[\bfvv(\bfxx_t,\bfxx_0)\mid\bfxx_t]$, then the deterministic one-sample estimator is unbiased: $\bbE\!\left[\widehat{\nabla V}_t^{\,\mathrm{det}}\mid\bfxx_t\right]=\nabla_{\bfxx_t}V_t(\bfxx_t,c)$;
    \item[(iii).] Under the assumption of (ii) and $h=1$, the KDE estimator is unbiased and has smaller variance than the deterministic one-sample estimator:
    \begin{align}
        \bbE\!\left[\widehat{\nabla V}_t^{\,\mathrm{KDE}}(\bfxx_t^j,c)\mid\bfxx_t^j\right]
        =\nabla_{\bfxx_t}V_t(\bfxx_t^j,c),\quad
        \mathrm{Var}\!\left(\widehat{\nabla V}_t^{\,\mathrm{KDE}}\mid\bfxx_t^j\right)
        \leq \mathrm{Var}\!\left(\widehat{\nabla V}_t^{\,\mathrm{det}}\mid\bfxx_t^j\right).
    \end{align}
\end{enumerate}
\end{proposition}
\begin{proof}
Again, we use the forward sampling time $r=1-t$ from the proof of \thmref{variance-reduction}.
\paragraph{Proof of (i).} Let $\clF_r$ be the filtration generated by the proposal process. For a step $\Delta r>0$, define
\begin{align}
\bfeps_{r,\Delta r}
:=\frac{\bfww_{r+\Delta r}-\bfww_r}{\sqrt{\Delta r}}.
\end{align}
Recall from \eqnref{app-dV} that the value martingale $V_r=\bbE_\bbQ[A(\bfyy_1,c)\mid\clF_r]$ satisfies
\begin{align}
\dd V_r
=g_r\nabla_{\bfxx}\bar V_r(\bfyy_r,c)\cdot\dd\bfww_r,\qquad
g_r=\sqrt{\frac{2t\eta_t}{1-t}},\quad t=1-r.
\end{align}
By the tower property and the fact that $\bfww_{r+\Delta r}-\bfww_r$ has zero conditional mean,
\begin{align}
\bbE_\bbQ\!\left[
A(\bfyy_1,c)(\bfww_{r+\Delta r}-\bfww_r)\mid\clF_r\right]&=\bbE_\bbQ\!\left[\bbE[
A(\bfyy_1,c)(\bfww_{r+\Delta r}-\bfww_r)\mid\clF_{r+\Delta r}]\mid\clF_r\right]\\
&=\bbE_\bbQ\!\left[\bbE[
A(\bfyy_1,c)\mid\clF_{r+\Delta r}](\bfww_{r+\Delta r}-\bfww_r)\mid\clF_r\right]\\
&=\bbE_\bbQ\!\left[
V_{r+\Delta r}(\bfww_{r+\Delta r}-\bfww_r)\mid\clF_r\right]\\
&=\bbE_\bbQ\!\left[
(V_{r+\Delta r}-V_r)(\bfww_{r+\Delta r}-\bfww_r)\mid\clF_r\right]\\
&=g_r\nabla_{\bfxx}\bar V_r(\bfyy_r,c)\,\Delta r+o(\Delta r),
\end{align}
where the last equality follows from the It\^o isometry and the continuity of the martingale integrand. Consequently,
\begin{align}
\bbE_\bbQ\!\left[
A(\bfyy_1,c)\frac{\bfeps_{r,\Delta r}}{\sqrt{\Delta r}}
\ \middle|\ \clF_r\right]
=g_r\nabla_{\bfxx}\bar V_r(\bfyy_r,c)+o(1).
\end{align}
The stochastic estimator in \secref{method-designspace} can be written as
\begin{align}
\widehat{\nabla V}_{t,\Delta r}^{\,\mathrm{sto}}
=\frac{1}{g_r}\,
A(\bfyy_1,c)\frac{\bfeps_{r,\Delta r}}{\sqrt{\Delta r}}.
\end{align}
Taking $\Delta r\to0$ and using the Markov property, so that conditioning on $\clF_r$ reduces to conditioning on $\bfyy_r=\bfxx_t$, yields
\begin{align}
\lim_{\Delta r\to0}
\bbE_\bbQ\!\left[
\widehat{\nabla V}_{t,\Delta r}^{\,\mathrm{sto}}\mid\bfxx_t\right]
=\nabla_{\bfxx_t}V_t(\bfxx_t,c).
\end{align}
Identifying the positive sampling-time step $\Delta r$ with the discretization magnitude $\Delta t$ in the proposition proves (i). Thus the stochastic one-sample estimator is asymptotically unbiased.

\paragraph{Proof of (ii).}
When $\eta_t\equiv1$, the transition kernel is $q(\bfxx_t\mid\bfxx_0)=\clN(\bfxx_t;(1-t)\bfxx_0,t^2\bfI)$. Under the assumption $\bfvv_\old(\bfxx_t,t)=\bbE[\bfvv(\bfxx_t,\bfxx_0)\mid\bfxx_t]$, by \propref{single-point},
\begin{align}
\nabla_{\bfxx_t}V_t(\bfxx_t,c)
=-\frac{1-t}{t}\,
\bbE\!\left[
A(\bfxx_0,c)
\big(\bfvv(\bfxx_t,\bfxx_0)-\bfvv_\old(\bfxx_t,t)\big)
\mid\bfxx_t\right].
\end{align}
The random vector inside this conditional expectation is precisely $\widehat{\nabla V}_t^{\,\mathrm{det}}$. Therefore
\begin{align}
\bbE\!\left[
\widehat{\nabla V}_t^{\,\mathrm{det}}\mid\bfxx_t\right]
=\nabla_{\bfxx_t}V_t(\bfxx_t,c),
\end{align}
which proves the unbiasedness of the deterministic one-sample estimator.

\paragraph{Proof of (iii).}
The results can be directly deduced from Prop. 2 of \citet{bertrand2026closed} (see also Exercise 7.32 of \citet{casella2024statistical}), we give the detailed proof here for completeness. Fix a query sample $(\bfxx_t^j,c)$ from the rollout group and write $q_0$ for the endpoint marginal under the proposal. At $h=1$, the kernel in \eqnref{kde-estimator} is proportional to the likelihood $q(\bfxx_t^j\mid\bfxx_0)$; hence Bayes' rule gives
\begin{align}
q(\bfxx_0\mid\bfxx_t^j)
=\frac{K(\bfxx_t^j,\bfxx_0)q_0(\bfxx_0)}
{\bbE_{\bfzz\sim q_0}[K(\bfxx_t^j,\bfzz)]}.
\end{align}
Conditioned on $\bfxx_t^j$, the anchor endpoint $\bfxx_0^j$ is therefore distributed according to $q(\bfxx_0\mid\bfxx_t^j)$, while the remaining $G-1$ endpoints are independent samples from $q_0$. To expose the resulting Rao--Blackwell structure, introduce an auxiliary index $J\sim\mathrm{Unif}\{1,\ldots,G\}$, sample $\bfX_J$ from $q(\bfxx_0\mid\bfxx_t^j)$, and sample $\{\bfX_i\}_{i\ne J}$ independently from $q_0$. Because the KDE estimator is symmetric in the group elements, this random relabeling does not change its distribution.

Define
\begin{align}
F(\bfX_i)
:=-\frac{1-t}{t}A(\bfX_i,c)
\big(\bfvv(\bfxx_t^j,\bfX_i)-\bfvv_\old(\bfxx_t^j,t)\big),
\qquad K_i:=K(\bfxx_t^j,\bfX_i).
\end{align}
Then by Bayes' rule,
\begin{align}
\bbP(J=i\mid\bfX_1,\ldots,\bfX_G)
=\frac{K_i}{\sum_{\ell=1}^G K_\ell}.
\end{align}
Since the KDE estimator is the conditional expectation of the deterministic anchor estimator:
\begin{align}
\bbE\!\left[
F(\bfX_J)\mid\bfX_1,\ldots,\bfX_G\right]
=\frac{\sum_{i=1}^G K_iF(\bfX_i)}
{\sum_{i=1}^G K_i}
=\widehat{\nabla V}_t^{\,\mathrm{KDE}}(\bfxx_t^j,c).
\end{align}
Applying the tower property and the result of (ii),
\begin{align}
\bbE\!\left[
\widehat{\nabla V}_t^{\,\mathrm{KDE}}(\bfxx_t^j,c)
\mid\bfxx_t^j\right]
=\bbE\!\left[F(\bfX_J)\mid\bfxx_t^j\right]
=\nabla_{\bfxx_t}V_t(\bfxx_t^j,c),
\end{align}
which means the KDE estimator is unbiased under the condition. Finally, applying the property of conditional variance, we get
\begin{align}
\mathrm{Var}\!\left(F(\bfX_J)\mid\bfxx_t^j\right)
&=\bbE\!\left[
\mathrm{Var}\!\left(F(\bfX_J)\mid\bfX_{1:G},\bfxx_t^j\right)
\middle|\bfxx_t^j\right]+
\mathrm{Var}\!\left(
\bbE[F(\bfX_J)\mid\bfX_{1:G},\bfxx_t^j]
\middle|\bfxx_t^j\right)\\
&\geq
\mathrm{Var}\!\left(
\widehat{\nabla V}_t^{\,\mathrm{KDE}}(\bfxx_t^j,c)
\middle|\bfxx_t^j\right).
\end{align}
Since $F(\bfX_J)$ has the same conditional distribution as $\widehat{\nabla V}_t^{\,\mathrm{det}}$, then
\begin{align}
\mathrm{Var}\!\left(
\widehat{\nabla V}_t^{\,\mathrm{KDE}}\mid\bfxx_t^j\right)
\leq
\mathrm{Var}\!\left(
\widehat{\nabla V}_t^{\,\mathrm{det}}\mid\bfxx_t^j\right).
\end{align}
\end{proof}

\section{Results in \tabref{master}}
\label{appx:table-proof}

In this section, we provide detailed derivations of results in \tabref{master}. 

\subsection{Flow-GRPO, GRPO-Guard and TempFlow-GRPO}

\paragraph{Flow-GRPO~\citep{liu2025flow}.} The stochastic-integral path-space estimator of the Flow-GRPO objective is \eqnref{sde-estimator}. Applying an Euler--Maruyama discretization to \eqnref{sde-estimator} gives
\begin{align}\label{eqn:appx-flowgrpo-1}
    -\clL^{\text{GRPO}}(\theta,c)\,&=_\theta\,\bbE_q\!\left[\,{A(\bfxx_0,c)\!\int_0^1\!\sqrt{2\wwt_{1-r}}\,\Delta\bfvv_\theta\cdot\dd\bfww_r}+{A(\bfxx_0,c)\!\int_0^1\!\wwt_{1-r}\,\lrVert{\Delta\bfvv_\theta}^2\dd r}\right] \\
    &\approx \mathbb{E}_q \Bigg[ 
\sum_{i=0}^{N-1} \Big( 
A(\mathbf{x}_0,c)\sqrt{2 \tilde{w}_{1-r_i}} \, \Delta \mathbf{v}_\theta \cdot  \sqrt{\Delta r_i}\bfeps_i
+ A(\mathbf{x}_0,c)\tilde{w}_{1-r_i} \, \|\Delta \mathbf{v}_\theta\|^2 \, \Delta r_i
\Big) \Bigg]\\
&= \mathbb{E}_q \Bigg[  
\sum_{i=0}^{N-1} \Big( A(\mathbf{x}_0,c)
\sqrt{2 \tilde{w}_{1-r_i}} \, \Delta \mathbf{v}_\theta \cdot  \frac{1}{\sqrt{\Delta r_i}}\bfeps_i
+ A(\mathbf{x}_0,c)\tilde{w}_{1-r_i} \, \|\Delta \mathbf{v}_\theta\|^2 \, \Big)\Delta r_i
 \Bigg]
\end{align}

With the change of variable $t=1-r$, let
$\Delta r_i=r_{i+1}-r_i=t_i-t_{i+1}=\Delta t_i > 0$.
The sum can then be written entirely in terms of $t$ as
\begin{align}\label{eqn:appx-flowgrpo-t}
    -\clL^{\text{GRPO}}(\theta,c)
    &\approx \mathbb{E}_q \Bigg[
    \sum_{i=0}^{N-1}\Big(
    A(\mathbf{x}_0,c)\sqrt{2\tilde{w}_{t_i}}\,
    \Delta\mathbf{v}_{\theta}\cdot\frac{1}{\sqrt{\Delta t_i}}\bfeps_i
    +A(\mathbf{x}_0,c)\tilde{w}_{t_i}\,
    \lrVert{\Delta\mathbf{v}_{\theta}}^2\Big)\Delta t_i
    \Bigg].
\end{align}
Substituting $\wwt_t=\frac{(1+\eta_t)^2}{4\eta_t}\frac{1-t}{t}$ into
\eqnref{appx-flowgrpo-t} directly gives the Flow-GRPO expressions in
\tabref{master}:
\begin{align}
    \widehat{\nabla V}_t^{\,\mathrm{sto}}
    &=s(t)A(\mathbf{x}_0,c)\bfeps,
    \quad s(t)=\sqrt{\frac{1-t}{2\eta_t t\Delta t}},
    \\
    \hat w_1(A,t)&=A\frac{(1+\eta_t)^2}{4\eta_t}\frac{1-t}{t},
    \quad \hat w_2(t)=1+\eta_t.
\end{align}

\paragraph{GRPO-Guard~\citep{wang2025grpo}.}
The difference between Flow-GRPO and GRPO-Guard is that GRPO-Guard drops the regularization Term~(B) and normalizes Term~(A) in \eqnref{sde-estimator}, corresponding to changes in weight $\hat w_1(A,t)$ and $\hat w_2(t)$:
\begin{align}
    \hat w_1(A,t)=0,
    \quad
    \hat w_2(t)=(1+\eta_t)\sqrt{\frac{2\eta_t t}{(1-t)\Delta t}}.
\end{align}

\paragraph{TempFlow-GRPO~\citep{he2025tempflow}.}
TempFlow-GRPO first rolls out an ODE trajectory. At each point $\bfxx_t$ along this trajectory, it takes a single SDE step to generate a group of branch states $\bfxx_{t-\Delta t}$, each of which is then rolled out to the endpoint using an ODE solver. The resulting trajectories form a branching structure, which is the reason we call it branched-SDE sampler. TempFlow-GRPO also reweights the policy loss by $g_t\sqrt{\Delta t}$, where $g_t=\sqrt{2\eta_t t/(1-t)}$, thus its weights are
\begin{align}
    \hat w_1(A,t)
    =A\frac{(1+\eta_t)^2}{4}\sqrt{\frac{2(1-t)\Delta t}{\eta_t t}},\quad
    \hat w_2(t)
    =(1+\eta_t)\sqrt{\frac{2\eta_t t\Delta t}{1-t}}.
\end{align}

\subsection{AWM and DiffusionNFT}

\paragraph{Advantage Weighted Matching (AWM)~\citep{xue2025advantage}.}
Recall from \secref{background-rl} that AWM first draws a rollout sample $\bfxx_0\sim p_\old(\cdot\mid c)$ (with Flow-SDE or Flow-ODE, based on different rewards) and then constructs an intermediate state by forward noising:
\begin{align}
    \bfxx_t=(1-t)\bfxx_0+t\bfeps,
    \quad \bfeps\sim\clN(\bfzro,\bfI).
\end{align}
The training objective is \footnote{The original AWM objective additionally includes an importance-sampling ratio between the current and rollout policies. This ratio is exactly one in the on-policy setting used by the practical algorithm and remains close to one under the mildly off-policy updates considered in practice. We therefore omit it from the present derivation.}
\begin{align}\label{eqn:appx-awm-original}
    -\clL^{\mathrm{AWM}}(\theta)
    =\bbE_{c,\bfxx_0\sim p_\old,\,t\sim p(t),\,\bfeps\sim\clN(\bfzro,\bfI)}\!\left[
        A(\bfxx_0,c)
        \lrVert{\bfvv_\theta(\bfxx_t,t,c)-\bfvv(\bfxx_t,\bfxx_0)}^2
    \right].
\end{align}

In our design space, we consider the reverse sampling process \eqnref{bg-samplingsde} to obtain $\bfxx_{[0:1]}$. To close the gap between the forward and reverse process, we construct the reverse sampling process with
\begin{align}    \bfvv_\base(\bfxx_t,t)=\bbE[\bfvv(\bfxx_t,\bfxx_0)\mid\bfxx_t]
~~\text{with}~~ \eta_t=1, 
    \quad \text{where}~~\bfvv(\bfxx_t,\bfxx_0):=\frac{\bfxx_t-\bfxx_0}{t}.
\end{align}
Under such choices, the reverse sampling process \eqnref{bg-samplingsde} is exactly the time reversal of the forward noising path above, which means that $(\bfxx_0, \bfxx_t)$ constructed by the reverse sampling process and forward noising process share the same joint distributions. Thus, the AWM is equivalent to using $\bfvv_\base(\bfxx_t,t)=\bbE[\bfvv(\bfxx_t,\bfxx_0)\mid\bfxx_t]$ and $\eta_t\equiv1$ in our design space. Under this condition, \propref{single-point} gives

\begin{align}\label{eqn:appx-awm-value-estimator}
\widehat{\nabla V}_t^{\,\mathrm{det}}(\bfxx_t,c)
    =-s(t)A(\bfxx_0,c)\big(\bfvv(\bfxx_t,\bfxx_0)-\bfvv_\old(\bfxx_t,t)\big),
    \quad s(t):=\frac{1-t}{t}.
\end{align}
We specify the weights
\begin{align}
    \hat w_1(A,t)=\frac{A}{\Delta t},
    \quad
    \hat w_2(t)=\frac{2}{s(t)\Delta t}.
\end{align}
Substituting these weights and the estimator \eqnref{appx-awm-value-estimator} into \eqnref{unified-loss} gives
\begin{align}
    -\clL^{\mathrm{policy}}(\theta)
    &=\int_0^1\frac{1}{\Delta t}\bbE_q\!\left[
        A\lrVert{\Delta\bfvv_\theta}^2
        -2A
        \left\langle\Delta\bfvv_\theta,\bfvv-\bfvv_\old\right\rangle
    \right]\dd t\\
    &=_\theta\int_0^1\frac{1}{\Delta t}\bbE_q\!\left[
        A(\bfxx_0,c)
        \lrVert{\bfvv_\theta(\bfxx_t,t,c)-\bfvv(\bfxx_t,\bfxx_0)}^2
    \right]\dd t,\label{eqn:appx-awm-continuous}
\end{align}.

For the $\Delta t$ term, partition $[0,1]$ into $n$ timestep bins $I_i$ of widths $\Delta t_i$, and write $\Delta t(t)=\Delta t_i$ for $t\in I_i$. We define the density by
\begin{align}\label{eqn:appx-awm-time-density}
    p(t):=\frac{1}{n\Delta t_i},\quad t\in I_i.
\end{align}
Using $1/\Delta t(t)=np(t)$, \eqnref{appx-awm-continuous} becomes
\begin{align}
    -\clL^{\mathrm{policy}}(\theta)
    &=_\theta n\,\bbE_{c,\bfxx_0\sim p_\old,\,t\sim p(t),\,\bfeps\sim\clN(\bfzro,\bfI)}\!\left[
        A(\bfxx_0,c)
        \lrVert{\bfvv_\theta(\bfxx_t,t,c)-\bfvv(\bfxx_t,\bfxx_0)}^2
    \right].
\end{align}
This is the same as the AWM loss \eqnref{appx-awm-original} up to the positive constant factor $n$, which does not change the optimization process. The timesteps sampled by the original AWM objective from $p(t)$ are precisely the points used to discretize the continuous-time integral in our formulation. 

\paragraph{DiffusionNFT~\citep{zheng2025diffusionnft}.}
Similar to AWM, DiffusionNFT first rolls out with the DPM-Solver~\citep{lu2025dpm} to obtain $\bfxx_0$ and then samples $\bfxx_t=(1-t)\bfxx_0+t\bfeps$. It minimizes the objective
\begin{align}\label{eqn:appx-nft-original}
    -\clL^{\mathrm{NFT}}(\theta)
    =\bbE_{c,\,\bfxx_0\sim p_\old,\,\bfxx_t\sim p(\cdot\mid\bfxx_0),\,t}\!\left[
        \hat R(\bfxx_0,c)\lrVert{\bfvv^+-\bfvv}^2
        +\big(1-\hat R(\bfxx_0,c)\big)\lrVert{\bfvv^--\bfvv}^2
    \right],
\end{align}
where
\begin{align}
    \bfvv^+=(1-\beta)\bfvv_\old(\bfxx_t,t,c)+\beta\bfvv_\theta(\bfxx_t,t,c)&,
    \quad
    \bfvv^-=(1+\beta)\bfvv_\old(\bfxx_t,t,c)-\beta\bfvv_\theta(\bfxx_t,t,c), \\
    \hat R(\bfxx_0,c)=\frac{A(\bfxx_0,c)+1}{2}&, \quad \bfvv=(\bfxx_t-\bfxx_0)/t.
\end{align}
Define
\begin{align}
    \Delta\bfvv_\theta:=\bfvv_\theta-\bfvv_\old,
    \quad
    \bfdd:=\bfvv-\bfvv_\old.
\end{align}
Then $\bfvv^+-\bfvv=\beta\Delta\bfvv_\theta-\bfdd$ and $\bfvv^--\bfvv=-\beta\Delta\bfvv_\theta-\bfdd$. Substituting these identities and expanding the two squares gives
\begin{align}
    -\clL^{\mathrm{NFT}}(\theta)
    &=\bbE\!\left[
        \hat R\lrVert{\beta\Delta\bfvv_\theta-\bfdd}^2
        +(1-\hat R)\lrVert{-\beta\Delta\bfvv_\theta-\bfdd}^2
    \right]\\
    &=\bbE\!\left[
        \beta^2\lrVert{\Delta\bfvv_\theta}^2
        -2\beta(2\hat R-1)\Delta\bfvv_\theta\cdot\bfdd
        +\lrVert{\bfdd}^2
    \right]\\
    &=_\theta\beta^2\bbE_{c,\,\bfxx_0\sim p_\old,\,\bfxx_t\sim p(\cdot\mid\bfxx_0),\,t}\!\left[
        \lrVert{\Delta\bfvv_\theta}^2
        -\Delta\bfvv_\theta\cdot\frac{2A}{\beta}(\bfvv-\bfvv_\old)
    \right],\label{eqn:appx-nft-expanded}
\end{align}
The scaling factor $\beta^2$ does not change the optimization process; thus, we have
\begin{align}\label{eqn:appx-nft-training}
    -{\clL}^{\mathrm{NFT}}(\theta)
    =\bbE_{c,\,\bfxx_0\sim p_\old,\,\bfxx_t\sim p(\cdot\mid\bfxx_0),\,t}\!\left[
        \lrVert{\Delta\bfvv_\theta}^2
        -\Delta\bfvv_\theta\cdot\frac{2A}{\beta}(\bfvv-\bfvv_\old)
    \right].
\end{align}
Therefore, similarly to AWM, DiffusionNFT is recovered in our design space by choosing $\bfvv_\base=\bbE[\bfvv\mid\bfxx_t]$, $\eta_t=1$, and the deterministic one-sample estimator \eqnref{appx-awm-value-estimator}. The weights functions are specified by $\hat w_1(t)=1/\Delta t$ and $\hat w_2(t)=2/(\beta s(t)\Delta t)$. 

\section{High-dimensional Behavior of the Path-space IS Ratio}
\label{appx:is-ratio}
We highlight an interesting phenomenon regarding the importance sampling (IS) ratio in \propref{policy-grad} under the proposal path measure. Although the IS ratio $\exp(M_1(\theta^-))$ has an expectation of $1$, it remains strictly below $1$ with high probability in high-dimensional latent spaces. \figref{appx_e_toy} illustrates this phenomenon concretely.

\begin{figure}[t]
  \centering
  \includegraphics[width=0.92\textwidth]{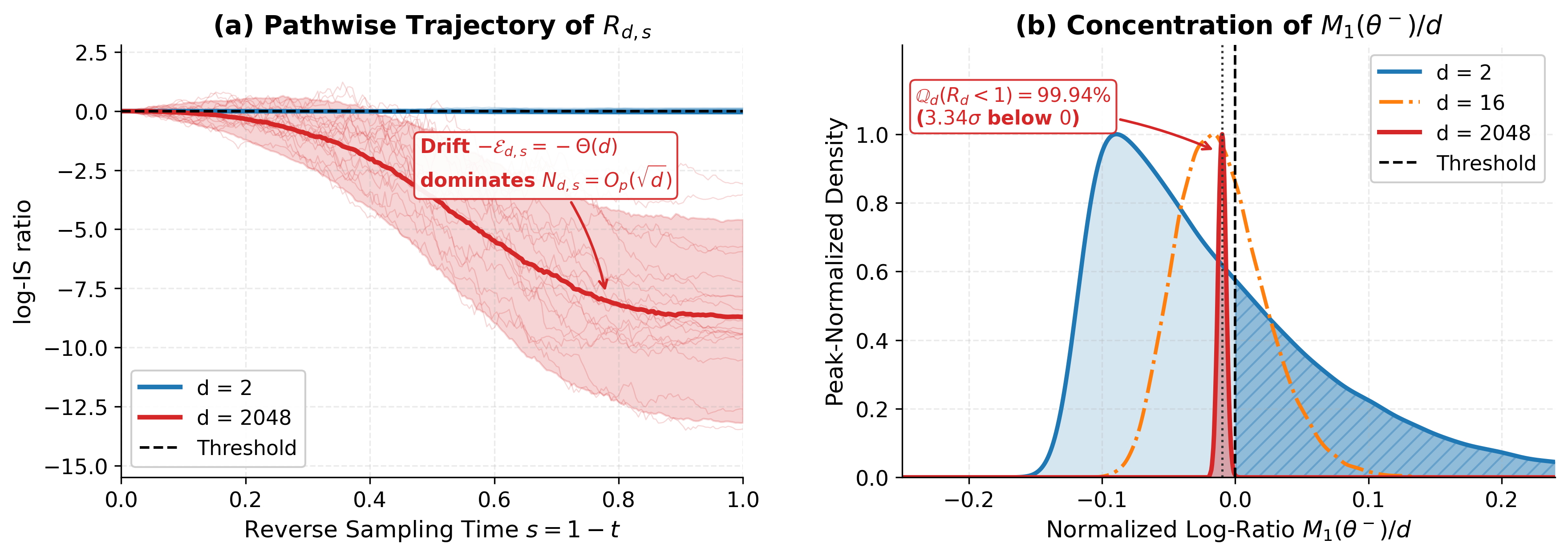}
  \caption{High-dimensional concentration of the path-space IS ratio $R_d = \exp(M_1(\theta^-))$ under the Flow-Matching SDE.
  \textbf{(a)} Along the reverse sampling time $s = 1-t \in [0, 1]$, the log-IS ratio for $d=2$ fluctuates around $0$ ($R_{2,s}=1$), whereas for $d=2048$, the negative linear drift dominates the martingale term and drives the trajectories exponentially below $R_{d,s}=1$.
  \textbf{(b)} Peak-normalized density of the normalized log-ratio $M_1(\theta^-)/d$, which concentrates sharply around $-\bar{\kappa} = -0.010 < 0$ as the dimension $d$ grows from $2$ to $2048$.}
  \label{fig:appx_e_toy}
\end{figure}

Fix a prompt $c$ and a parameter value $\theta^-$. All
probabilities below are conditional on this fixed prompt and policy. 
For each dimension $d$, let $\mathbb Q_d$ be the proposal path measure in \appxref{setup}, and let $\bfww_r$ be its $d$-dimensional Brownian motion in sampling time $r=1-t$. Denote
\begin{align}
 \Delta\mathbf v_d(\mathbf x,t)
 &:=\mathbf v_{\theta^-}(\mathbf x,t,c)
       -\mathbf v_{\mathrm{base}}(\mathbf x,t,c),
 &\widetilde w_t&:=\frac{(1+\eta_t)^2}{4\eta_t}\frac{1-t}{t},
 \label{app:is:delta}\\
 \mathbf h_{d,r}
 &:=\sqrt{2\widetilde w_{1-r}}\,
          \Delta\mathbf v_d(\mathbf x_{1-r},1-r),
 &N_{d,s}&:=-\int_0^s\mathbf h_{d,r}\cdot\mathrm d\bfww_r,
 \label{app:is:martingale}\\
 \mathcal E_{d,s}
 &:=\frac12\int_0^s\|\mathbf h_{d,r}\|^2\,\mathrm dr,
 &\mathcal E_d&:=\mathcal E_{d,1}.
 \label{app:is:energy}
\end{align}
Assume $\eta_t>0$ for $t\in(0,1)$, and $\mathcal E_d<\infty$
$\mathbb Q_d$-almost surely. Then the IS ratio is given by
\begin{align}
 R_d:=\exp\bigl(M_1(\theta^-)\bigr)
 &=\frac{\mathrm d\mathbb P_{\theta^-,d}}{\mathrm d\mathbb Q_d},
 & M_1(\theta^-)&=N_{d,1}-\mathcal E_d,
 &\langle N_d\rangle_s&=2\mathcal E_{d,s}.
 \label{eqn:is-decomposition}
\end{align}
Let $R_{d,s}:=\exp(N_{d,s}-\mathcal E_{d,s})$, so that $R_{d,1}=R_d$.
The proposition below shows the concentration of the IS ratio:

\begin{proposition}
\label{prop:is-tail}
In the setup above, assume the Novikov condition $\mathbb E_{\mathbb Q_d}[e^{\mathcal E_d}]<\infty$.
Then for every
$\lambda\in[0,1]$, the process
\begin{align}
 Z_{d,s}^{(\lambda)}
 &:
 =\exp\!\left(\lambda N_{d,s}-\lambda^2\mathcal E_{d,s}\right),
 \qquad 0\leq s\leq1,
 \label{app:is:exponential}
\end{align}
is a uniformly integrable $\mathbb Q_d$-martingale on $[0,1]$, with
$\mathbb E_{\mathbb Q_d}[Z_{d,s}^{(\lambda)}]=1$ for every $s\in[0,1]$.
Furthermore, for every $a>0$ and $\gamma\in[0,1)$,
\begin{align}
 \mathbb Q_d\!\left(R_d\geq e^{-\gamma a},\ \mathcal E_d\geq a\right)
 &\leq \exp\!\left(-\frac{(1-\gamma)^2a}{4}\right).
 \label{eqn:is-jointbound}
\end{align}
Consequently,
\begin{align}
 \mathbb Q_d(R_d\geq1)
 &\leq \mathbb Q_d(\mathcal E_d<a)+e^{-a/4}.
 \label{eqn:is-signbound}
\end{align}
\end{proposition}

\begin{proof}
Fix $\lambda\in[0,1]$. The quadratic variation of $\lambda N_d$ is
$\langle\lambda N_d\rangle_s=\lambda^2\langle N_d\rangle_s$.
Consequently,
\begin{align}
 \mathbb E_{\mathbb Q_d}\!\left[
  \exp\!\left(\frac12\langle\lambda N_d\rangle_1\right)\right]
 &=\mathbb E_{\mathbb Q_d}[e^{\lambda^2\mathcal E_d}]
 \leq\mathbb E_{\mathbb Q_d}[e^{\mathcal E_d}]<\infty.
 \label{app:is:scalednovikov}
\end{align}
Thus $Z_d^{(\lambda)}$ is a uniformly integrable martingale on $[0,1]$.
Since $Z_{d,0}^{(\lambda)}=1$, its expectation equals one at every time.
Its relation to the original importance ratio is
\begin{align}
 Z_{d,s}^{(\lambda)}
 &=R_{d,s}^{\lambda}
   \exp\!\left(\lambda(1-\lambda)\mathcal E_{d,s}\right).
 \label{app:is:relation}
\end{align}
On
$F:=\{R_d\geq e^{-\gamma a},\,\mathcal E_d\geq a\}$,
\eqnref{is-decomposition} gives
\begin{align}
 N_{d,1}
 &\geq\mathcal E_d-\gamma a
 \geq(1-\gamma)\mathcal E_d,\\
 \log Z_{d,1}^{(\lambda)}
 &\geq\bigl(\lambda(1-\gamma)-\lambda^2\bigr)\mathcal E_d \geq \bigl(\lambda(1-\gamma)-\lambda^2\bigr)a, \quad \forall\lambda\in[0,1].
\end{align}
Therefore
\begin{align}
 \mathbb Q_d(F)
 &\leq e^{-\bigl(\lambda(1-\gamma)-\lambda^2\bigr)a}
      \mathbb E_{\mathbb Q_d}
       \!\left[Z_{d,1}^{(\lambda)}\mathbf 1_F\right]
 \leq e^{-\bigl(\lambda(1-\gamma)-\lambda^2\bigr)a}
       \mathbb E_{\mathbb Q_d}[Z_{d,1}^{(\lambda)}]
 =e^{-\bigl(\lambda(1-\gamma)-\lambda^2\bigr)a}.
\end{align}
Take $\lambda=(1-\gamma)/2$ maximizes
$\lambda(1-\gamma)-\lambda^2$ over $[0,1]$, we have
\begin{align}
    \mathbb Q_d(F)
 \leq \exp\!\left(-\frac{(1-\gamma)^2a}{4}\right).
\end{align}
Taking $\gamma=0$ and splitting
according to whether $\mathcal E_d<a$ proves \eqnref{is-signbound}.

\end{proof}

\begin{corollary}[High-dimensional suppression of the importance ratio]
\label{app:is:dimension}
Under the assumptions of \propref{is-tail}, suppose there exist
a constant $\kappa>0$ independent of $d$ and a sequence
$\delta_d\to0$ such that
\begin{align}
 \mathbb Q_d\!\left(
   \int_0^1\widetilde w_t
       \|\Delta\mathbf v_d(\mathbf x_t,t)\|^2\,\mathrm dt
       \geq\kappa d\right)
 =\mathbb Q_d(\mathcal E_d\geq\kappa d)
 \geq1-\delta_d.
 \label{eqn:is-assumption}
\end{align}
Then
\begin{align}
 \mathbb Q_d\!\left(\exp(M_1(\theta^-))<1\right)
 &\geq1-\delta_d-e^{-\kappa d/4}
 \longrightarrow1.
 \label{app:is:conclusion}
\end{align}
\end{corollary}

\begin{proof}
Apply \propref{is-tail} with $a=\kappa d$, and use
\eqnref{is-assumption} to bound the event
$\{\mathcal E_d<\kappa d\}$.
\end{proof}

\section{Related Work}
\label{appx:related}
\paragraph{Reverse-sampling-based RL.}
DDPO~\citep{black2023training} and DPOK~\citep{fan2023dpok} pioneered policy-gradient fine-tuning for DDPM-style models by treating reverse denoising as a finite-horizon MDP. Flow-GRPO~\citep{liu2025flow} and Dance-GRPO~\citep{xue2025dancegrpo} extend this paradigm to flow-matching models by converting the sampling ODE into a marginally equivalent SDE and applying clipped GRPO-style updates. Subsequent methods refine this recipe along different axes: GRPO-Guard~\citep{wang2025grpo} regulates and rescales per-step likelihood ratios, TempFlow-GRPO~\citep{he2025tempflow} branches trajectories at selected timesteps, and MixGRPO~\citep{li2025mixgrpo} confines stochastic exploration to a sliding time window. FDFO~\citep{mcallister2026finite} follows a different route and estimates improvement directions through finite differences rather than analytic policy gradients. Our framework disentangles the estimator, weighting, and sampling choices underlying these methods and explains their efficiency differences through the resulting value-gradient variance and scale.

Complementary analyses also seek to organize this rapidly growing design space. \citet{choi2026rethinking} separate the policy-gradient objective, likelihood estimator, and sampler, and identify likelihood estimation as the dominant empirical factor. We refine the objective axis further by decomposing it into value-gradient estimation and the associated on- and off-policy weights. The concurrent work Reward Score Matching~\citep{lee2026reward} develops a score-space view that connects several reward-based fine-tuning objectives; in contrast, our path-space formulation explicitly recovers the training structures of AWM and DiffusionNFT and relates them to reverse-sampling policy gradients.

\paragraph{Forward-matching RL methods.}
An orthogonal line retains the analytic forward-noising process and injects reward information into a pretraining-style regression loss. Reward-Weighted Regression (RWR)~\citep{lee2023aligning} weights the diffusion training objective by endpoint rewards. AWM~\citep{xue2025advantage} and DiffusionNFT~\citep{zheng2025diffusionnft} similarly avoid the reverse-process MDP and construct reward- or advantage-dependent flow-matching targets from sampled clean endpoints and their analytically noised states. These methods report substantial improvements in training efficiency, but their relationship to reverse-sampling policy gradients and the role of their specialized timestep weights were previously unclear. Our path-space derivation shows that they instantiate the same policy-gradient structure with a deterministic value-gradient estimator and different weight choices.

\paragraph{Stochastic-control and adjoint-based methods.}
A complementary line formulates KL-regularized reward fine-tuning as entropy-regularized stochastic optimal control. ELEGANT~\citep{uehara2024fine} develops this viewpoint for continuous-time diffusion models, while Adjoint Matching~\citep{domingo2025adjoint} introduces a memoryless noise schedule and turns the Bellman optimality condition into an on-policy regression of the value gradient; its pathwise estimator, however, requires reward gradients and a backward adjoint sweep along an SDE trajectory. Tilt Matching~\citep{potaptchik2025tilt} instead relates the original and reward-tilted velocity fields through a regression objective and avoids both reward gradients and trajectory backpropagation by annealing through intermediate tilted models. Most closely related to our work, the concurrent Reinforce Adjoint Matching (RAM) method~\citep{bergmeister2026reinforce} takes a stochastic optimal control viewpoint, identifies the value gradient as the optimal adjoint, and combines the resulting fixed point with a REINFORCE identity to derive a forward-matching-type loss. In contrast, we derive the value-gradient form from path-space policy-gradient variance reduction, yielding a unified framework for both GRPO-style reverse-sampling algorithms and forward-based methods such as AWM and DiffusionNFT. Our framework further establishes an estimator-specific optimal weight mapping and introduces a new multi-sample KDE value-gradient estimator in the final recipe.

\paragraph{Reward-gradient and offline preference methods.}
When the reward is differentiable, it can be optimized directly by backpropagating through all or part of the sampling trajectory~\citep{prabhudesai2023aligning,xu2023imagereward,clark2023directly}. This approach provides a direct training signal but requires reward gradients and incurs substantial memory and computation costs when differentiation spans many denoising steps. A separate offline line adapts direct preference optimization to diffusion models, learning from paired human-preference data without online reward evaluation~\citep{wallace2024diffusion,yang2024using,yuan2024self}. These methods operate under a different feedback regime from the black-box online rewards considered in our framework.

\section{Experimental Details}
\label{appx:experimental-details}

\subsection{Additional Results}\label{appx:add-results}
\begin{figure}[htbp]
    \centering

    \begin{minipage}{0.98\textwidth}
        \centering
    
        \begin{subfigure}[b]{0.32\textwidth}
            \centering
            \includegraphics[width=\linewidth]{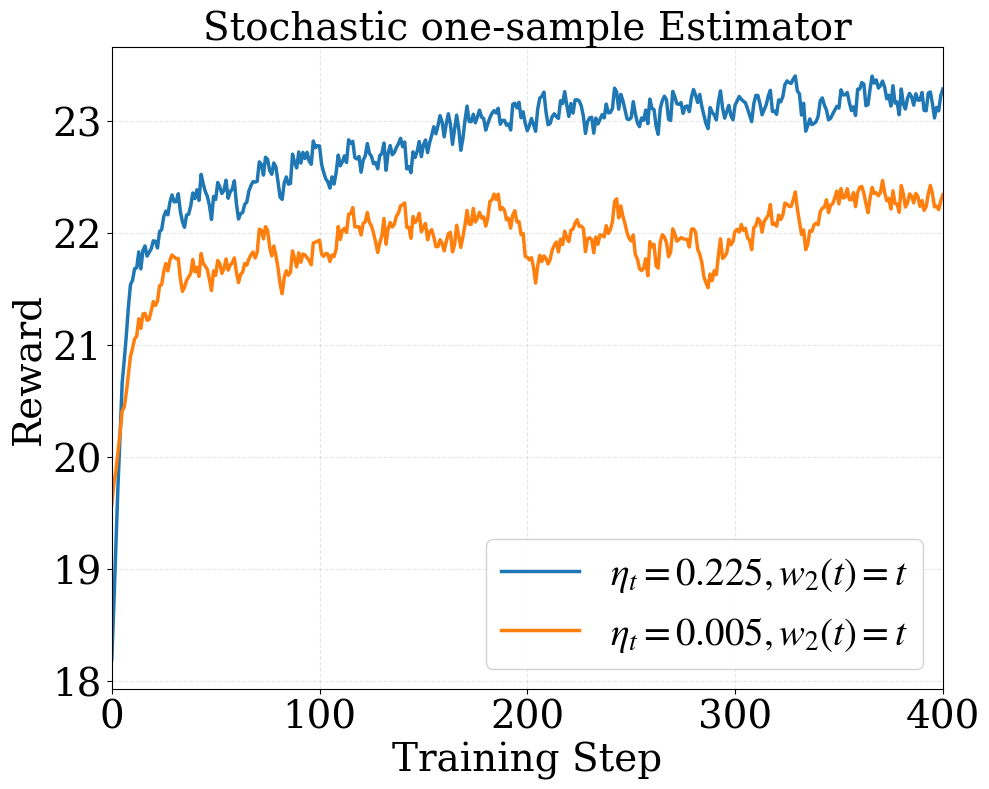}
            \caption{}
            \label{fig:1}
        \end{subfigure}%
        \hfill
        \begin{subfigure}[b]{0.32\textwidth}
            \centering
            \includegraphics[width=\linewidth]{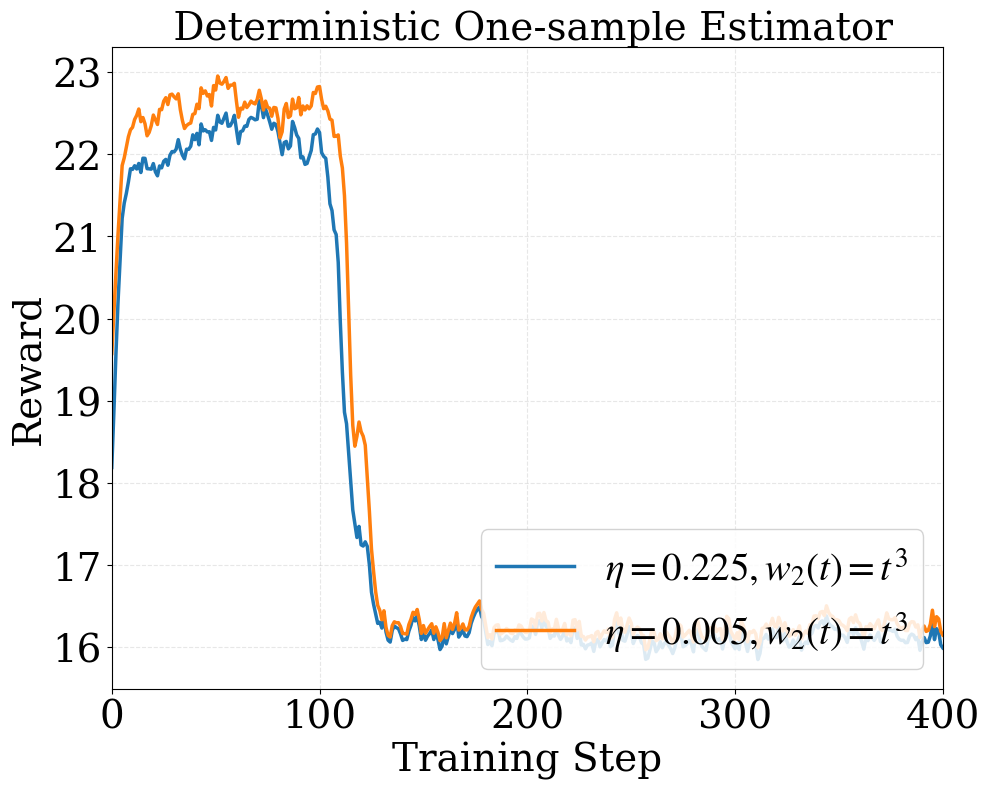}
            \caption{}
            \label{fig:2}
        \end{subfigure}%
        \hfill
        \begin{subfigure}[b]{0.32\textwidth}
            \centering
            \includegraphics[width=\linewidth]{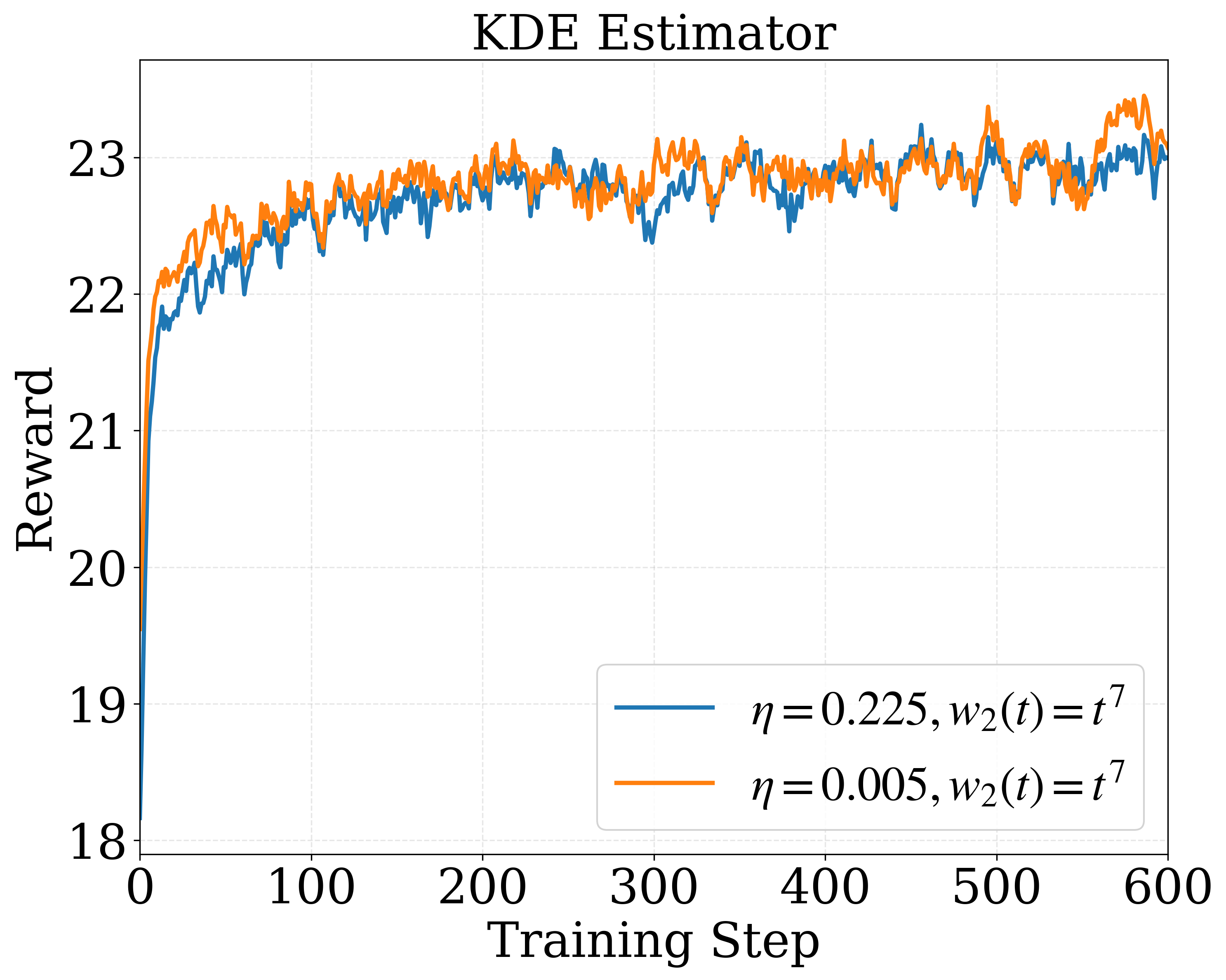}
            \caption{}
            \label{fig:3}
        \end{subfigure}        
    \end{minipage}
    \caption{Sampling noise schedule ablations for the stochastic one-sample estimator, deterministic one-sample estimator, and our KDE estimator. For each estimator, we run experiments for sampler noise schedule $\eta_t\in\{0.005,0.225\}$ on the Pickscore reward. }
    \label{fig:noise-scale}

\end{figure}

\begin{figure}[htbp]
    \centering

    \begin{minipage}{0.67\textwidth}
        \centering
    
        \begin{subfigure}[b]{0.49\textwidth}
            \centering
            \includegraphics[width=\linewidth]{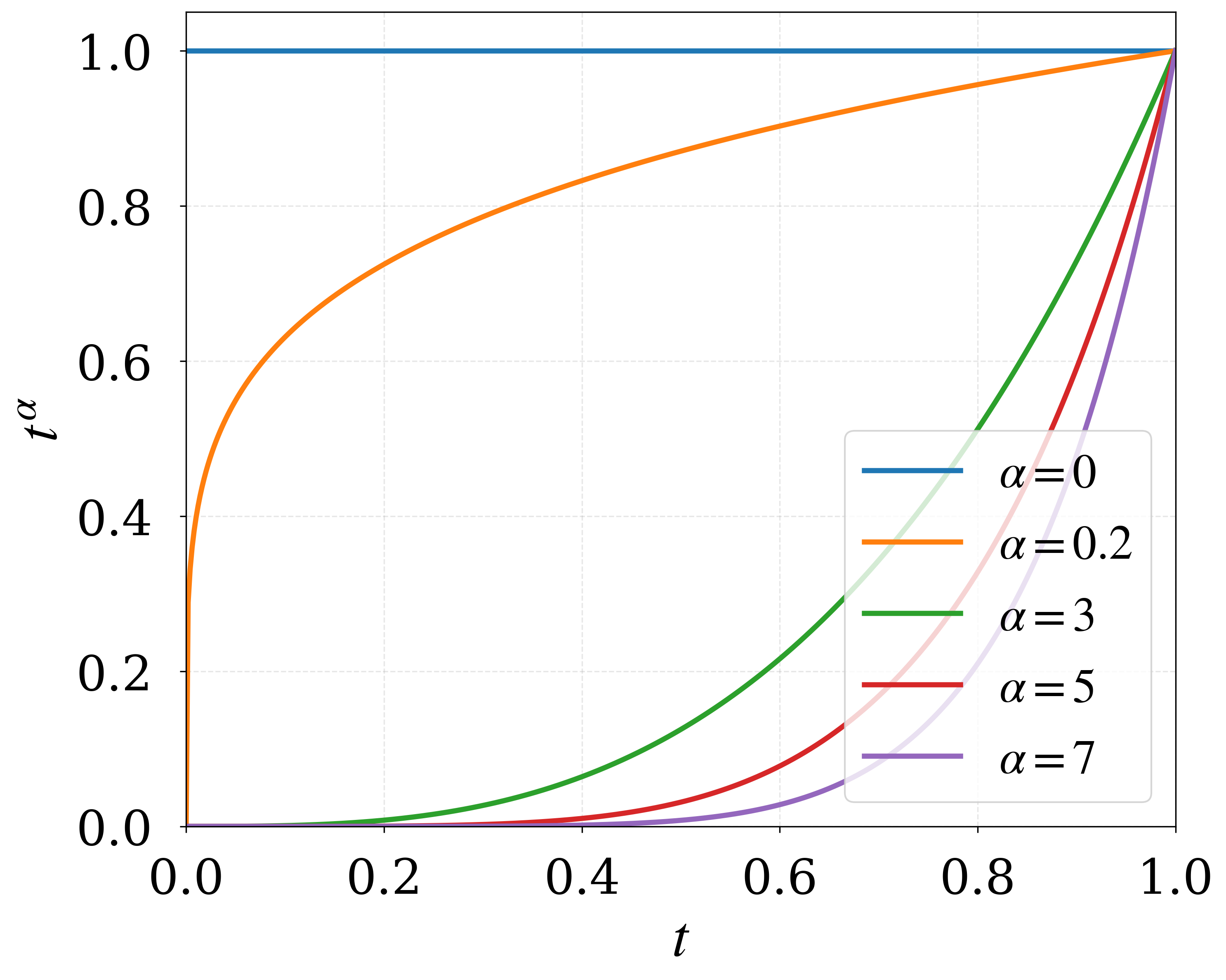}
            \caption{}
            \label{fig:1}
        \end{subfigure}%
        \hfill
        \begin{subfigure}[b]{0.49\textwidth}
            \centering
            \includegraphics[width=\linewidth]{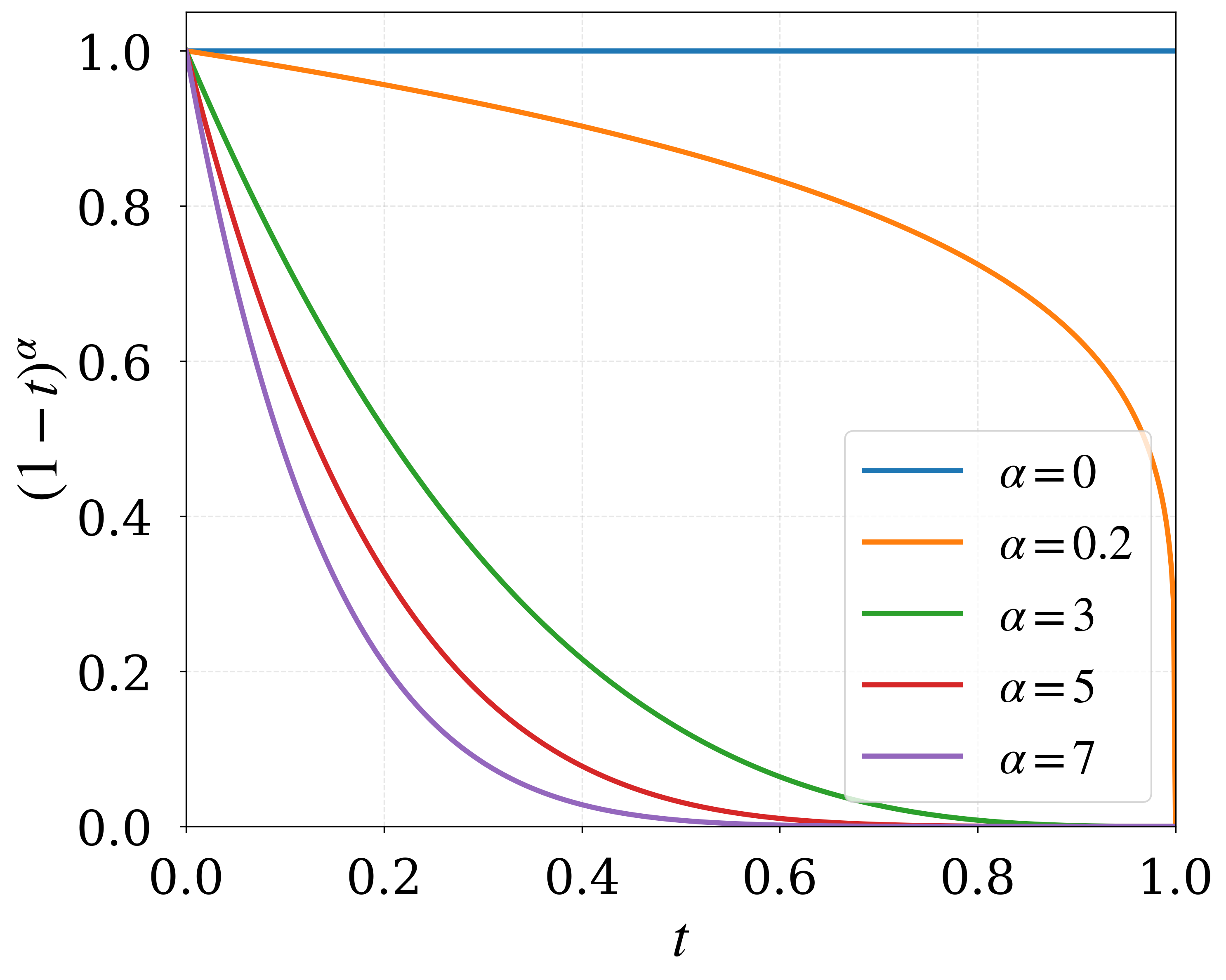}
            \caption{}
            \label{fig:2}
        \end{subfigure}%
        \hfill
          
    \end{minipage}
    \caption{Visualization of different weights}
    \label{fig:weight-visual}
\end{figure}
In \figref{noise-scale}, we show the sampling noise schedule ablations for each estimator. We find that the stochastic estimator in GRPO-based methods performs better in the larger $\eta_t$ setting (\figref{noise-scale} (a)), while the deterministic estimator suffers from training instability in both settings, and the small noise scheduler performs slightly better (\figref{noise-scale} (b)). Our estimator is not sensitive to different sampling noise scales and is stable across all these settings.

In \figref{weight-visual}, we provide a visualization of different weight functions.

\subsection{Hyperparameters}
We summarize the hyperparameters of our experiments in \tabref{hyp}. In all the experiments, we disable CFG and use the batch-level advantage normalization.

\begin{table}[h]
\centering\small
\caption{Hyperparameters used throughout \secref{exp}.}\label{tab:hyp}
\begin{tabular}{@{}ll@{}}
\toprule
Hyperparameter & Value \\
\midrule
Group size $G$ & $24$ for SD 3.5, $16$ for Qwen-Image \\
Train denoising steps $T$ & $10$ \\
Inference denoising steps & $40$ for SD 3.5, $50$ for Qwen-Image\\
Rollout SDE noise schedule & 0.225 for stochastic estimator, 0.005 for deterministic and KDE estimator \\
Optimizer & AdamW, $\mathrm{lr}=3\!\times\!10^{-4}$, $\beta_1\!=\!0.9,\beta_2\!=\!0.999$ \\
KL coefficient $\beta$ & $10^{-4}$ \\
Batch size & $9$ for SD 3.5, $4$ for Qwen-Image \\
Number of batches per epoch & $16$ for SD 3.5, $4$ for Qwen-Image \\
Train CFG & Disable CFG in training \\ 
$h$ in \eqnref{kde-estimator} & $d/2000$, where $d$ is the latent dimension\\
\bottomrule
\end{tabular}
\end{table}

\subsection{Robustness across random seeds}
\label{appx:seed-results}

We repeat the SD3.5-Medium experiment with PickScore reward using three random seeds while keeping all other configurations fixed. \tabref{three-seeds} reports the mean and standard deviation at four matched training checkpoints.

\begin{table}[h]
\centering\small
\caption{PickScore across three random seeds (mean $\pm$ standard deviation). Our method achieves a higher mean at every checkpoint with small run-to-run variation.}
\label{tab:three-seeds}
\setlength{\tabcolsep}{5pt}
\resizebox{\textwidth}{!}{%
\begin{tabular}{@{}lcccc@{}}
\toprule
Training time (GPU hours) &
$18.58$ &
$36.44$ &
$54.28$ &
$72.12$ \\
\midrule
DiffusionNFT &
$22.85\pm0.02886$ &
$23.06\pm0.01618$ &
$23.16\pm0.02153$ &
$23.22\pm0.01274$ \\
Ours &
$22.95\pm0.05240$ &
$23.28\pm0.02239$ &
$23.48\pm0.005547$ &
$23.56\pm0.01461$ \\
\bottomrule
\end{tabular}
}
\end{table}

At the final checkpoint, our method obtains $23.56\pm0.01461$, corresponding to a relative standard deviation of $0.062\%$; DiffusionNFT obtains $23.22\pm0.01274$ with a relative standard deviation of $0.055\%$. The consistent gap across checkpoints and seeds indicates that the improvement is not driven by a favorable random seed.

\subsection{Runtime comparison}
\label{appx:runtime}

The KDE estimator computes pairwise distances in $O(G^2d)$ time. With $G=24$ and latent dimension $d=65{,}536$, this is on the order of $10^7$ arithmetic operations per timestep and is efficiently parallelized on GPUs. It also reuses cached rollout latents, velocities, and rewards, so it requires no additional denoiser evaluation. Under matched settings, our update uses a trainable-policy forward and, when KL regularization is enabled, a frozen-reference forward. DiffusionNFT additionally evaluates the old policy at the renoised latent, requiring three rather than two denoiser forwards per optimized timestep. Since one denoiser forward costs approximately $10^{12}$ FLOPs, this saved network evaluation dominates the pairwise-distance overhead.

\begin{table}[h]
\centering
\caption{Measured wall-clock training time per epoch under matched SD3.5-Medium settings.}
\label{tab:runtime}
\begin{tabular}{@{}lc@{}}
\toprule
Method & Training time per epoch (seconds) \\
\midrule
DiffusionNFT & $115$ \\
KDE estimator (ours) & $97$ \\
\bottomrule
\end{tabular}
\end{table}

As shown in \tabref{runtime}, our method saves 18 seconds per epoch, a $15.7\%$ wall-clock reduction. Thus, the $O(G^2d)$ KDE computation has little effect on end-to-end runtime, while eliminating one denoiser forward pass provides a pronounced speedup.

\section{Qualitative Examples}
\label{appx:qualitative}
The visualization grids (\figref{case1}, \figref{case2},\figref{case3}) contain prompts from the test set of Pickscore reward training. For each prompt, images are generated with the same inference configuration from SD3.5-Medium after the RL training with GRPO-Guard, DiffusionNFT, and our method. The rows are organized by prompt split and the columns by method, enabling direct inspection of prompt alignment, perceptual quality, excessive saturation, and repetitive reward-hacking patterns without selecting examples solely by reward. 

As a complementary quantitative check, a panel of eight evaluators assesses randomly generated images drawn from a pool of 30 prompts spanning the same three splits. Evaluators are blinded to method identity and judge overall visual quality, prompt alignment, and aesthetic preference. Aggregating the evaluation decisions according to this protocol yields method-level preference shares of $51.4\%$ for our method, $29.4\%$ for GRPO-Guard, and $19.2\%$ for DiffusionNFT. These results suggest that the reward gains do not merely reflect numerical over-optimization.

\begin{figure}[t]
    \makebox[\linewidth]{
        \IfFileExists{Fig/paper_case_1.pdf}{%
            \includegraphics[width=\linewidth]{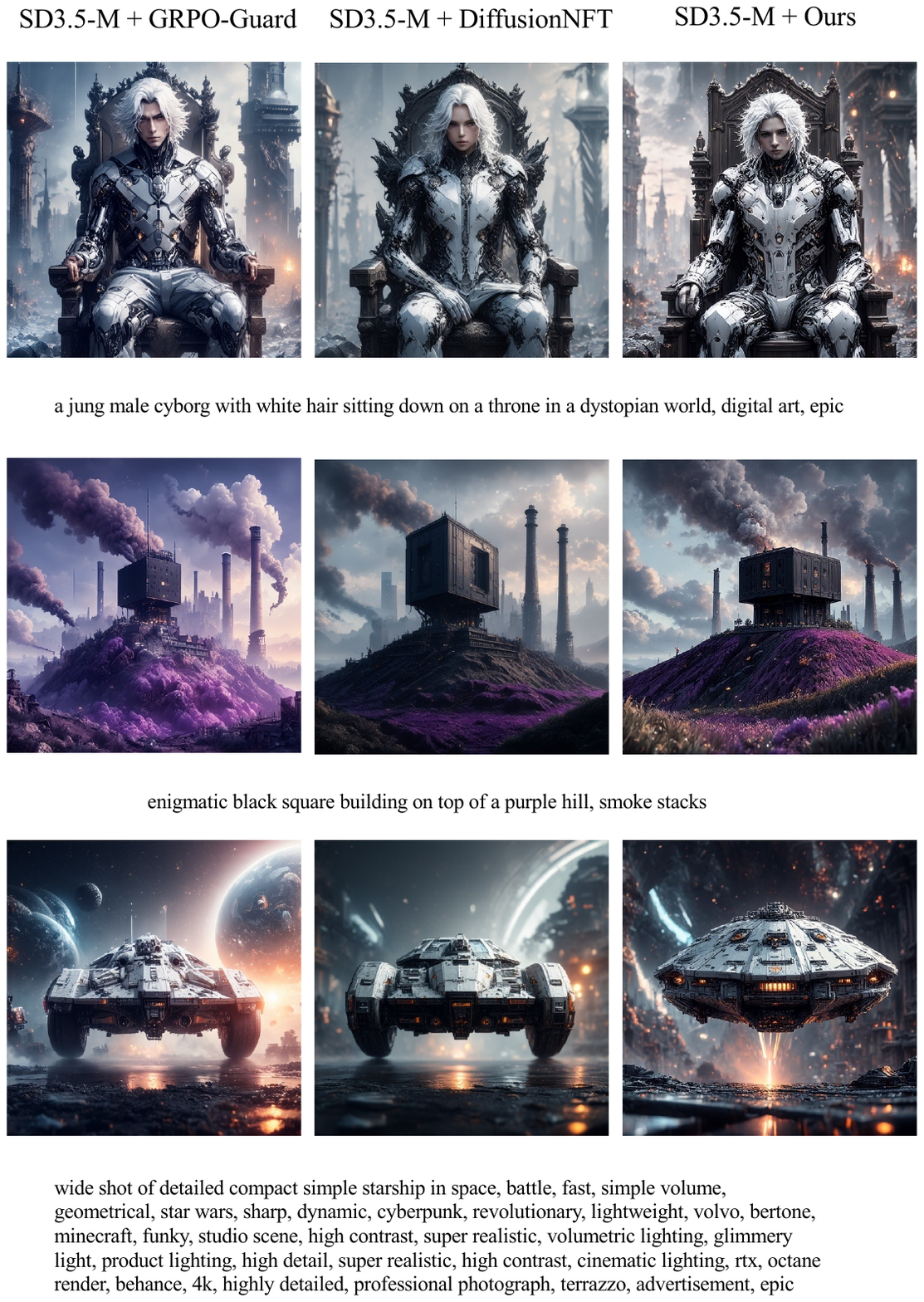}%
        }{%
            \fbox{\parbox[c][35mm][c]{0.95\linewidth}{\centering
            Placeholder for \texttt{Fig/paper\_case\_1.pdf}}}%
        }
    }
    \caption{Qualitative comparison across GRPO-Guard, DiffusionNFT and our method.}
    \label{fig:case1}
\end{figure}
\begin{figure}[t]
    \makebox[\linewidth]{
        \IfFileExists{Fig/paper_case_2.pdf}{%
            \includegraphics[width=\linewidth]{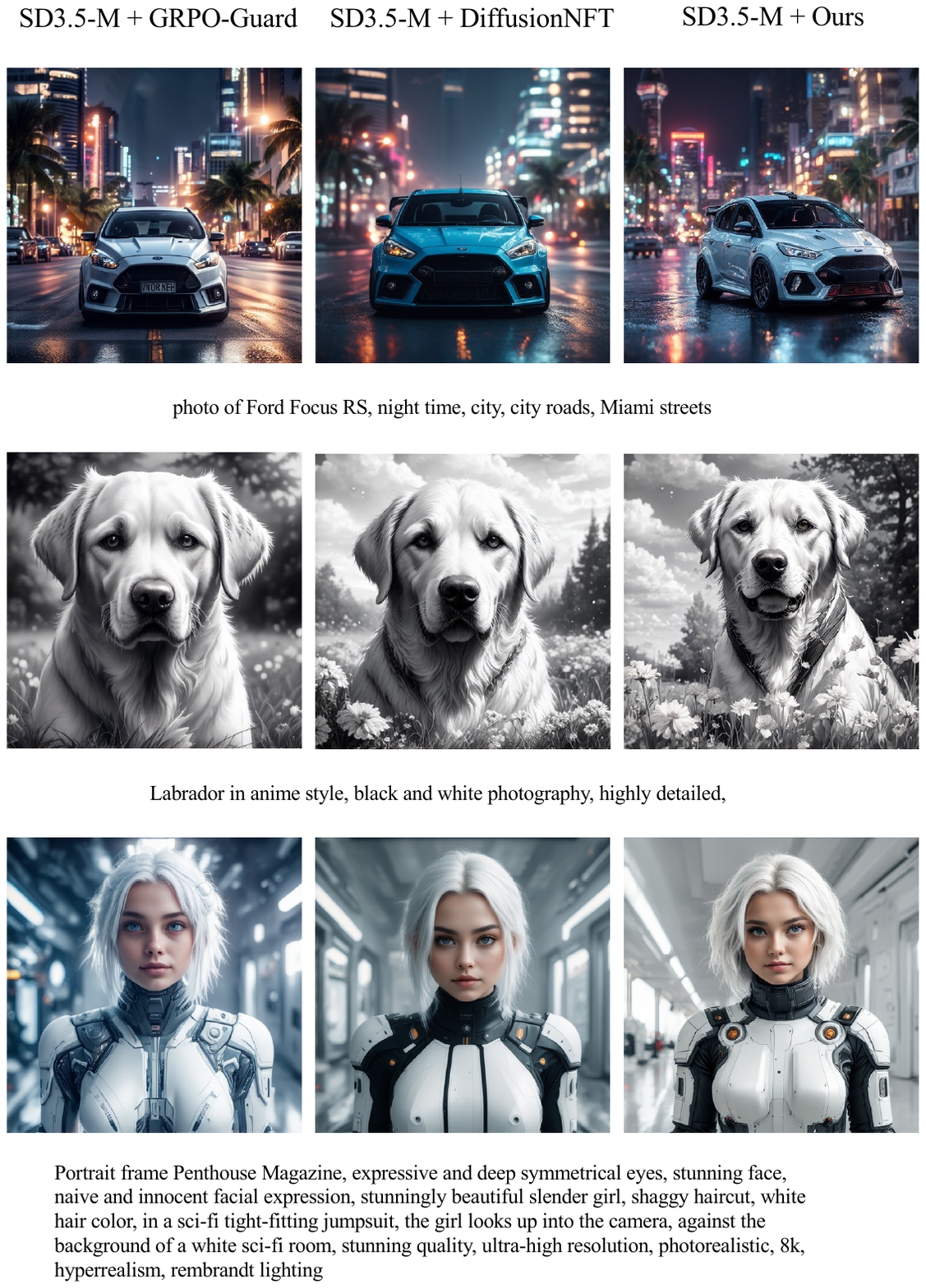}%
        }{%
            \fbox{\parbox[c][35mm][c]{0.95\linewidth}{\centering
            Placeholder for \texttt{Fig/paper\_case\_2.pdf}}}%
        }
    }
    \caption{Qualitative comparison across GRPO-Guard, DiffusionNFT and our method.}
    \label{fig:case2}
\end{figure}
\begin{figure}[t]
    \makebox[\linewidth]{
        \IfFileExists{Fig/paper_case_3.pdf}{%
            \includegraphics[width=\linewidth]{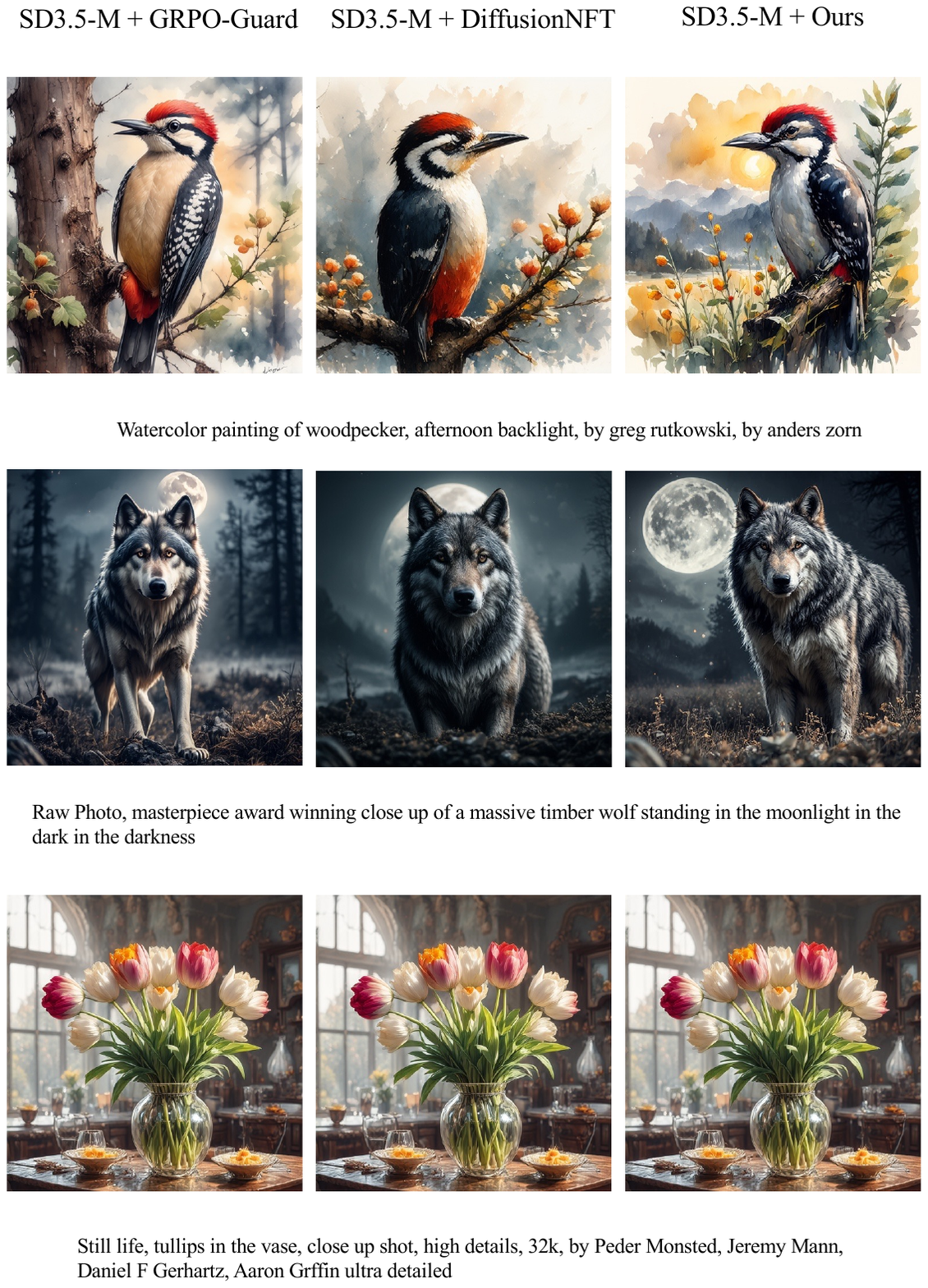}%
        }{%
            \fbox{\parbox[c][35mm][c]{0.95\linewidth}{\centering
            Placeholder for \texttt{Fig/paper\_case\_3.pdf}}}%
        }
    }
    \caption{Qualitative comparison across GRPO-Guard, DiffusionNFT and our method.}
    \label{fig:case3}
\end{figure}

\section{Limitations}
Our experiments are conducted on SD3.5-M and Qwen-Image with a fixed group size $G=24$ and a 
limited set of reward functions (PickScore, OCR, and GenEval). While these cover a representative 
range of reward types, the generalizability of our findings to other model architectures 
(e.g., video diffusion models), reward functions (\emph{e.g.}, 
human-preference models or verifiable task rewards) has not been systematically evaluated.

\section{Social Impact}\label{appx:social-impact}

This work proposes a unified theoretical framework for reinforcement learning post-training 
of diffusion generative models, which can facilitate the alignment of generative models 
with human preferences and safety constraints, thereby contributing to the responsible 
development of generative AI systems. However, this work mainly focuses on theory and algorithms, so the work does not have a broader social impact.


\newpage

\end{document}